%% file: linearSEM.tex
\documentclass[11pt,twoside]{article}

\usepackage{jmlr2e_modified}

\usepackage{amsfonts,amsmath,amssymb}
\usepackage{graphics,graphicx}
\usepackage{algorithm}
\usepackage{algpseudocode}

\usepackage[colorlinks]{hyperref}

\newcommand{\scriptM}{\ensuremath{\mathcal{M}}}
\newcommand{\opnorm}[1]{\left|\!\left|\!\left|{#1}\right|\!\right|\!\right|}

\newcommand{\condind}{\ensuremath{\perp\!\!\!\perp}}
\newcommand{\Pa}{\ensuremath{\operatorname{Pa}}}
\newcommand{\E}{\ensuremath{\operatorname{E}}}
\newcommand{\tr}{\ensuremath{\operatorname{tr}}}
\newcommand{\score}{\ensuremath{\operatorname{score}}}
\newcommand{\Sigmahat}{\ensuremath{\widehat{\Sigma}}}
\newcommand{\Thetahat}{\ensuremath{\widehat{\Theta}}}
\newcommand{\Gamhat}{\ensuremath{\widehat{\Gamma}}}

\newcommand{\bhat}{\ensuremath{\widehat{b}}}
\newcommand{\scorehat}{\ensuremath{\widehat{\score}}}

\newcommand{\scriptU}{\ensuremath{\mathcal{U}}}

\newcommand{\Btil}{\ensuremath{\widetilde{B}}}
\newcommand{\scriptD}{\ensuremath{\mathcal{D}}}
\newcommand{\fhat}{\ensuremath{\widehat{f}}}
\newcommand{\mprob}{\ensuremath{\mathbb{P}}}
\newcommand{\ftil}{\ensuremath{\widetilde{f}}}
\newcommand{\scriptR}{\ensuremath{\mathcal{R}}}
\newcommand{\scriptC}{\ensuremath{\mathcal{C}}}
\newcommand{\Ghat}{\ensuremath{\widehat{G}}}
\newcommand{\atil}{\ensuremath{\widetilde{a}}}
\newcommand{\btil}{\ensuremath{\widetilde{b}}}
\newcommand{\Thetamin}{\ensuremath{\Theta_0^{\min}}}

\newtheorem{assumption}{Assumption}

\input{\texdir/RatbertMacros}

\ShortHeadings{Linear SEMs via inverse covariance estimation}{Loh and B\"{u}hlmann}
\firstpageno{1}

\begin{document}

\title{High-dimensional learning of linear causal networks via inverse covariance estimation}

\author{%
\name Po-Ling Loh \email ploh@berkeley.edu \\
\addr Department of Statistics\\
University of California\\
Berkeley, CA 94720, USA
\AND
\name Peter B\"{u}hlmann \email buhlmann@stat.math.ethz.ch\\
\addr Seminar f\"{u}r Statistik\\
ETH Z\"{u}rich\\
Switzerland
}

\editor{}

\maketitle

\begin{abstract}%   <- trailing '%' for backward compatibility of .sty file
	
We establish a new framework for statistical estimation of directed acyclic graphs (DAGs) when data are generated from a linear, possibly non-Gaussian structural equation model. Our framework consists of two parts: (1) inferring the moralized graph from the support of the inverse covariance matrix; and (2) selecting the best-scoring graph amongst DAGs that are consistent with the moralized graph. We show that when the error variances are known or estimated to close enough precision, the true DAG is the unique minimizer of the score computed using the reweighted squared $\ell_2$-loss. Our population-level results have implications for the identifiability of linear SEMs when the error covariances are specified up to a constant multiple. On the statistical side, we establish rigorous conditions for high-dimensional consistency of our two-part algorithm, defined in terms of a ``gap'' between the true DAG and the next best candidate. Finally, we demonstrate that dynamic programming may be used to select the optimal DAG in linear time when the treewidth of the moralized graph is bounded. 

\end{abstract}


\section{Introduction}

Causal networks arise naturally in a wide variety of application domains, including genetics, epidemiology, and time series analysis~\citep{HugEtal00, SteEtal12, AaEtal12}. The task of inferring the graph structure of a causal network from joint observations is a relevant but challenging problem. Whereas undirected graphical structures may be estimated via pairwise conditional independence testing, with worst-case time scaling as the square of the number of nodes, estimation methods for directed acyclic graphs (DAGs) first require learning an appropriate permutation order of the vertices, leading to computational complexity that scales exponentially in the graph size. Greedy algorithms present an attractive computationally efficient alternative, but such methods are not generally guaranteed to produce the correct graph~\citep{Chi02}. In contrast, exact methods for causal inference that search exhaustively over the entire DAG space may only be tractable for relatively small graphs~\citep{SilMyl06}.

\subsection{Restricted search space}

In practice, knowing prior information about the structure of the underlying DAG may lead to vast computational savings. For example, if a natural ordering of the nodes is known, inference may be performed by regressing each node upon its predecessors and selecting the best functional fit for each node. This yields an algorithm with runtime linear in the number of nodes and overall quadratic complexity. In the linear high-dimensional Gaussian setting, one could apply a version of the graphical Lasso, where the feasible set is restricted to matrices that are upper-triangular with respect to the known ordering~\citep{ShoMic10}. However, knowing the node order seems quite unrealistic in practice, except for certain special cases. If instead a conditional independence graph or superset of the skeleton is specified a priori, the number of required conditional independence tests may also be reduced dramatically. This appears to be a more reasonable assumption, and various authors have devised algorithms to compute the optimal DAG efficiently in settings where the input graph has bounded degree and/or bounded treewidth~\citep{PerEtal08, OrdSze12, KorPar13}.

Unfortunately, appropriate tools for inferring such superstructures are rather limited, and the usual method of using the graphical Lasso to estimate a conditional independence graph is rigorously justified only in the linear Gaussian setting~\citep{YuaLin07}. Recent results have established that a version of the graphical Lasso may also be used to learn a conditional independence graph for variables taking values in a discrete alphabet when the graph has bounded treewidth~\citep{LohWai12}, but results for more general distributions are absent from the literature. \cite{BuhEtal13} isolate sufficient conditions under which Lasso-based linear regression could be used to recover a conditional independence graph for general distributions, and use it as a prescreening step for nonparametric causal inference in additive noise models; however, it is unclear which non-Gaussian distributions satisfy such conditions.

\subsection{Our contributions}

We propose a new algorithmic strategy for inferring the DAG structure of a linear, potentially non-Gaussian structural equation model (SEM). Deviating slightly from the literature, we use the term \emph{non-Gaussian} to refer to the fact that the variables are not jointly Gaussian; however, we do \emph{not} require non-Gaussianity of all exogenous noise variables, as assumed by~\cite{ShiEtal11}. We proceed in two steps, where each step is of independent interest: First, we infer the moralized graph by estimating the inverse covariance of the joint distribution. The novelty is that we justify this approach for non-Gaussian linear SEMs. Second, we find the optimal causal network structure by searching over the space of DAGs that are consistent with the moralized graph and selecting the DAG that minimizes an appropriate score function. When the score function is decomposable and the moralized graph has bounded treewidth, the second step may be performed via dynamic programming in time linear in the number of nodes~\citep{OrdSze12}. Our algorithm is also applicable in a high-dimensional setting when the moralized graph is sparse, where we estimate the support of the inverse covariance matrix using a method such as the graphical Lasso~\citep{RavEtal11}. Our algorithmic framework is summarized in Algorithm~\ref{AlgDAG}:

\begin{algorithm}
	\caption{Framework for DAG estimation}
	\label{AlgDAG}
	\begin{algorithmic}[1]
		\State \textbf{Input:} Data samples $\{x_i\}_{i=1}^n$ from a linear SEM
		\vspace{.2cm}
		\State Obtain estimate $\Thetahat$ of inverse covariance matrix (e.g., using graphical Lasso)
		\State \hspace{\algorithmicindent} Construct moralized graph $\widehat{\scriptM}$ with edge set defined by $\supp(\Thetahat)$
		\State Compute scores for DAGs that are consistent with $\widehat{\scriptM}$ (e.g., using squared $\ell_2$-error)
		\State \hspace{\algorithmicindent} Find minimal-scoring $\Ghat$ (using dynamic programming when score is decomposable and $\widehat{\scriptM}$ has bounded treewidth)
		\vspace{.2cm}
		\State \textbf{Output:} Estimated DAG $\Ghat$
	\end{algorithmic}
\end{algorithm}

We prove the correctness of our graph estimation algorithm by deriving new results about the theory of linear SEMs. We present a novel result showing that for almost every choice of linear coefficients, the support of the inverse covariance matrix of the joint distribution is identical to the edge structure of the moralized graph. Although a similar relationship between the support of the inverse covariance matrix and the edge structure of an undirected conditional independence graph has long been established for multivariate Gaussian models~\citep{Lau96}, our result does not involve any assumptions of Gaussianity, and the proof technique is completely new.

Since we do not impose constraints on the error distribution of our SEM, standard parametric maximum likelihood methods are \emph{not} applicable to score and compare candidate DAGs. Consequently, we use the squared $\ell_2$-error to score DAGs, and prove that in the case of homoscedastic errors, the true DAG uniquely minimizes this score function. As a side corollary, we establish that the DAG structure of a linear SEM is identifiable whenever the additive errors are homoscedastic, which generalizes a recent result derived only for Gaussian variables~\citep{PetBuh13}. In addition, our result covers cases with Gaussian and non-Gaussian errors, whereas~\cite{ShiEtal11} require all errors to be non-Gaussian (see Section~\ref{SecIdentify}). A similar result is implicitly contained under some assumptions in~\cite{vanBuh13}, but we provide a more general statement and additionally quantify a regime where the errors may exhibit a certain degree of heteroscedasticity. Thus, when errors are not too heteroscedastic, the much more complicated ICA algorithm~\citep{ShiEtal06, ShiEtal11} may be replaced by a simple scoring method using squared $\ell_2$-loss.

On the statistical side, we show that our method produces consistent estimates of the true DAG by invoking results from high-dimensional statistics. We note that our theoretical results only require a condition on the gap between squared $\ell_2$-scores for various DAGs in the restricted search space and eigenvalue conditions on the true covariance matrix, which is much weaker than the restrictive beta-min condition from previous work~\citep{vanBuh13}. Furthermore, the size of the gap is \emph{not} required to scale linearly with the number of nodes in the graph, unlike similar conditions in~\cite{vanBuh13} and~\cite{PetBuh13}, leading to genuinely high-dimensional results. Although the precise size of the gap relies heavily on the structure of the true DAG, we include several examples giving intuition for when our condition could be expected to hold (see Sections~\ref{SecSmallDAGs} and~\ref{SecGap} below). Finally, since inverse covariance matrix estimation and computing scores based on linear regression are both easily modified to deal with systematically corrupted data~\citep{LohWai11a}, we show that our methods are also applicable for learning the DAG structure of a linear SEM when data are observed subject to corruptions such as missing data and additive noise.

The remainder of the paper is organized as follows: In Section~\ref{SecBackground}, we review the general theory of probabilistic graphical models and linear SEMs. Section~\ref{SecCIG} describes our results on the relationship between the inverse covariance matrix and conditional independence graph of a linear SEM. In Section~\ref{SecFit}, we discuss the use of the squared $\ell_2$-loss for scoring candidate DAGs. Section~\ref{SecStat} establishes results for the statistical consistency of our proposed inference algorithms and explores the gap condition for various graphs. Finally, Section~\ref{SecComputation} describes how dynamic programming may be used to identify the optimal DAG in linear time, when the moralized graph has bounded treewidth. Proofs of supporting results are contained in the Appendix.

%%%%%%%%%%%

\section{Background}
\label{SecBackground}

We begin by reviewing some basic background material and introducing notation for the graph estimation problems studied in this paper.

\subsection{Graphical models}
\label{SecGM}

In this section, we briefly review the theory of directed and undirected graphical models, also known as conditional independence graphs (CIGs). For a more in-depth exposition, see~\cite{Lau96} or~\cite{KolFri09} and the references cited therein.

\subsubsection{Undirected graphs}
\label{SecUGM}

Consider a probability distribution $q(x_1, \dots, x_p)$ and an undirected graph $G = (V,E)$, where $V = \{1, \dots, p\}$ and $E \subseteq V \times V$. We say that $G$ is a \emph{conditional independence graph} (CIG) for $q$ if the following \emph{Markov condition} holds: For all disjoint triples $(A, B, S) \subseteq V$ such that $S$ separates $A$ from $B$ in $G$, we have $X_A \condind X_B \mid X_S$. Here, $X_C \defn \{X_j: j \in C\}$ for any subset $C \subseteq V$. We also say that $G$ \emph{represents} the distribution $q$.

By the well-known Hammersley-Clifford theorem, if $q$ is a strictly positive distribution (i.e., $q(x_1, \dots, x_p) > 0$ for all $(x_1, \dots, x_p)$), then $G$ is a CIG for $q$ if and only if we may write
\begin{equation}
	\label{EqnUndFact}
	q(x_1, \dots, x_p) = \prod_{C \in \scriptC} \psi_C(x_C),
\end{equation}
for some potential functions $\{\psi_C: C \in \scriptC\}$ defined over the set of cliques $\scriptC$ of $G$. In particular, note that the complete graph on $p$ nodes is always a CIG for $q$, but CIGs with fewer edges often exist.

\subsubsection{Directed acyclic graphs (DAGs)}
\label{SecDAG}

Changing notation slightly, consider a \emph{directed} graph $G = (V,E)$, where we now distinguish between edges $(j,k)$ and $(k,j)$. We say that $G$ is a \emph{directed acyclic graph} (DAG) if there are no directed paths starting and ending at the same node. For each node $j \in V$, let $\Pa(j) \defn \{k \in V: (k,j) \in E\}$ denote the \emph{parent set} of $j$, where we sometimes write $\Pa_G(j)$ to emphasize the dependence on $G$. A DAG $G$ \emph{represents} a distribution $q(x_1, \dots, x_p)$ if $q$ factorizes as
\begin{equation}
	\label{EqnDAGFact}
	q(x_1, \dots, x_p) \; \propto \; \prod_{j=1}^p q(x_j \mid q_{\Pa(j)}).
\end{equation}
Finally, a permutation $\pi$ of the vertex set $V = \{1, \dots, p\}$ is a \emph{topological order} for $G$ if $\pi(j) < \pi(k)$ whenever $(j,k) \in E$. Such a topological order exists for any DAG, but it may not be unique. The factorization~\eqref{EqnDAGFact} implies that $X_j \condind X_{\nu(j)} \mid X_{\Pa(j)}$ for all $j$, where $\nu(j) \defn \{k \in V: (j,k) \not\in E \text{ and } k \not\in \Pa(j)\}$ is the set of all nondescendants of $j$ excluding its parents.

Given a DAG $G$, we may form the \emph{moralized graph} $\scriptM(G)$ by fully connecting all nodes within each parent set $\Pa(j)$ and dropping the orientations of directed edges. Note that moralization is a purely graph-theoretic operation that transforms a directed graph into an undirected graph. However, if the DAG $G$ represents a distribution $q$, then $\scriptM(G)$ is also a CIG for $q$. This is because each set $\{j\} \cup \Pa(j)$ forms a clique $C_j$ in $\scriptM(G)$, and we may define the potential functions $\psi_{C_j}(x_{C_j}) \defn q(x_j \mid q_{\Pa(j)})$ to obtain the factorization~\eqref{EqnUndFact} from the factorization~\eqref{EqnDAGFact}.

Finally, we define the \emph{skeleton} of a DAG $G$ to be the undirected graph formed by dropping orientations of edges in $G$. Note that the edge set of the skeleton is a subset of the edge set of the moralized graph, but the latter set is generally much larger. The skeleton is not in general a CIG.

\subsection{Linear structural equation models}
\label{SecSEM}

We now specialize to the framework of linear structural equation models.

We say that a random vector $X = (X_1, \dots, X_p) \in \real^p$ follows a \emph{linear structural equation model} (SEM) if
\begin{equation}
\label{EqnLinStruct}
	X = B^T X + \epsilon,
\end{equation}
where $B$ is a strictly upper triangular matrix known as the \emph{autoregression matrix}. We assume $\E[X] = \E[\epsilon] = 0$ and $\epsilon_j \condind (X_1, \dots, X_{j-1})$ for all $j$.

In particular, observe that the DAG $G$ with vertex set $V = \{1, \dots, p\}$ and edge set $E = \{(j,k): B_{jk} \neq 0\}$ represents the joint distribution $q$ on $X$. Indeed, equation~\eqref{EqnLinStruct} implies that
\begin{equation*}
	q(X_j \mid X_1, \dots, X_{j-1}) = q(X_j \mid X_{\Pa_G(j)}),
\end{equation*}
so we may factorize
\begin{equation*}
	q(X_1, \dots, X_p) = \prod_{j=1}^p q(X_j \mid X_1, \dots, X_{j-1}) = \prod_{j=1}^p q(X_j \mid X_{\Pa_G(j)}).
\end{equation*}
Given samples $\{X^i\}_{i=1}^n$, our goal is to infer the unknown matrix $B$, from which we may recover $G$ (or vice versa).

%%%%%%%%%%%%%%

\section{Moralized graphs and inverse covariance matrices}
\label{SecCIG}

In this section, we describe our main result concerning inverse covariance matrices of linear SEMs. It generalizes a result for multivariate Gaussians, and states that the inverse covariance matrix of the joint distribution of a linear SEM reflects the structure of a conditional independence graph.

We begin by noting that
\begin{equation*}
	\E[X_j \mid X_1, \dots, X_{j-1}] = b_j^T X,
\end{equation*}
where $b_j$ is the $j^\text{th}$ column of $B$, and
\begin{equation*}
%	\label{EqnLinCoeff}
	b_j = \left(\Sigma_{j, 1:(j-1)} \left(\Sigma_{1:(j-1), 1:(j-1)}\right)^{-1}, 0, \dots, 0\right)^T.
\end{equation*}
Here, $\Sigma \defn \cov[X]$. We call $b_j^T X$ the \emph{best linear predictor} for $X_j$ amongst linear combinations of $\{X_1, \dots, X_{j-1}\}$. Defining $\Omega \defn \cov[\epsilon]$ and $\Theta \defn \Sigma^{-1}$, we then have the following lemma, proved in Appendix~\ref{AppCovs}:
\begin{lemma}
	\label{LemCovs}
	The matrix of error covariances is diagonal: $\Omega = \diag(\sigma_1^2, \dots, \sigma_p^2)$ for some $\sigma_i > 0$. The entries of $\Theta$ are given by
	\begin{align}
		\label{EqnInvElts}
		\Theta_{jk} & = -\sigma_k^{-2} B_{jk} + \sum_{\ell > k} \sigma_\ell^{-2} B_{j \ell} B_{k \ell}, \qquad \forall j < k, \\
		\label{EqnInvDiag}
		\Theta_{jj} & = \sigma_j^{-2} + \sum_{\ell > j} \sigma_\ell^{-2} B_{j\ell}^2, \qquad \qquad \qquad \forall j.
	\end{align}
\end{lemma}

In particular, equation~\eqref{EqnInvElts} has an important implication for causal inference, which we state in the following theorem. Recalling the notation of Section~\ref{SecDAG}, the graph $\scriptM(G)$ denotes the moralized DAG.
\begin{theorem}
	\label{ThmInvDAG}
	Suppose $X$ is generated from the linear structural equation model~\eqref{EqnLinStruct}. Then $\Theta$ reflects the graph structure of the moralized DAG; i.e., for $j \neq k$, we have $\Theta_{jk} = 0$ if $(j,k)$ is not an edge in $\scriptM(G)$.
\end{theorem}

\begin{proof}
Suppose $j \neq k$ and $(j,k)$ is not an edge in $\scriptM(G)$, and assume without loss of generality that $j < k$. Certainly, $(j,k) \not\in E$, implying that $B_{jk} = 0$. Furthermore, $j$ and $k$ cannot share a common child, or else $(j,k)$ would be an edge in $\scriptM(G)$. This implies that either $B_{j\ell} = 0$ or $B_{k\ell} = 0$ for each $\ell > k$. The desired result then follows from equation~\eqref{EqnInvElts}.
\end{proof}

In the results to follow, we will assume that the converse of Theorem~\ref{ThmInvDAG} holds, as well. This is stated in the following Assumption:
\begin{assumption}
	\label{AsInvDAG}
	Let $(B, \Omega)$ be the matrices of the underlying linear SEM. For every $j < k$, we have
	\begin{equation*}
	-\sigma_k^{-2} B_{jk} + \sum_{\ell > k} \sigma_\ell^{-2} B_{j\ell} B_{k\ell} = 0
	\end{equation*}
	only if $B_{jk} = 0$ and $B_{j\ell} B_{k\ell} = 0$ for all $\ell > k$.
\end{assumption}	
Combined with Theorem~\ref{ThmInvDAG}, Assumption~\ref{AsInvDAG} implies that
$\Theta_{jk} = 0$ if and only if $(j,k)$ is not an edge in $\scriptM(G)$. (Since $\Theta \succ 0$, the diagonal entries of $\Theta$ are always strictly positive.) Note that when the nonzero entries of $B$ are independently sampled continuous random variables, Assumption~\ref{AsInvDAG} holds for all choices of $B$ except on a set of Lebesgue measure zero.

\begin{remark}
Note that Theorem~\ref{ThmInvDAG} may be viewed as an extension of the canonical result for Gaussian graphical models. Indeed, a multivariate Gaussian distribution may be written as a linear SEM with respect to any permutation order $\pi$ of the variables, giving rise to a DAG $G^\pi$. Theorem~\ref{ThmInvDAG} states that $\supp(\Theta)$ is always a subset of the edge set of $\scriptM(G^\pi)$. Assumption~\ref{AsInvDAG} is a type of \emph{faithfulness} assumption~\citep{KolFri09, SpiEtal00}. By selecting different topological orders $\pi$, one may then derive the usual result that $X_j \condind X_k \mid X_{\backslash \{j,k\}}$ if and only if $\Theta_{jk} = 0$, for the Gaussian setting. Note that this conditional independence assertion may not always hold for linear SEMs, however, since non-Gaussian distributions are not necessarily expressible as a linear SEM with respect to an arbitrary permutation order. Indeed, we only require Assumption~\ref{AsInvDAG} to hold with respect to a single (fixed) order.
\end{remark}

%%%%%%%%%%%

\section{Score functions for DAGs}
\label{SecFit}

Having established a method for reducing the search space of DAGs based on estimating the moralized graph, we now move to the more general problem of scoring candidate DAGs. As before, we assume the setting of a linear SEM.

Parametric maximum likelihood is often used as a score function for statistical estimation of DAG structure, since it enjoys the nice property that the population-level version is maximized only under a correct parametrization of the model class. This follows from the relationship between maximum likelihood and KL divergence:
\begin{equation*}
	\arg\max_\theta \E_{\theta_0}[\log p_\theta(X)] = \arg\min_\theta \E_{\theta_0} \left[\log\left(\frac{p_{\theta_0}(X)}{p_\theta(X)}\right)\right] = \arg\min_\theta D_{KL}(p_{\theta_0} \| p_\theta),
\end{equation*}
and the latter quantity is minimized exactly when $p_{\theta_0} \equiv p_\theta$, almost everywhere. If the model is identifiable, this happens only if $\theta = \theta_0$.

However, such maximum likelihood methods presuppose a fixed parametrization for the model class. In the case of linear SEMs, this translates into an appropriate parametrization of the error vector $\epsilon$. For comparison, note that minimizing the squared $\ell_2$-error for ordinary linear regression may be viewed as a maximum likelihood approach when errors are Gaussian, but the $\ell_2$-minimizer is still statistically consistent for estimation of the regression vector when errors are \emph{not} Gaussian. When our goal is recovery of the autoregression matrix $B$ of the DAG, it is therefore natural to ask whether squared $\ell_2$-error could be used in place of maximum likelihood as an appropriate metric for evaluating DAGs.

We will show that in settings when the noise variances $\{\sigma_j\}_{j=1}^p$ are specified up to a constant (e.g., homoscedastic error), the answer is affirmative. In such cases, the true DAG uniquely minimizes the $\ell_2$-loss. As a side result, we also show that the true linear SEM is identifiable.

\begin{remark}
\cite{NowBuh13} study the use of nonparametric maximum likelihood methods for scoring candidate DAGs. We remark that such methods could also be combined with the framework of Sections~\ref{SecCIG} and~\ref{SecComputation} to select the optimal DAG for linear SEMs with nonparametric error distributions: First, estimate the moralized graph via the inverse covariance matrix, and then find the DAG with minimal score using a method such as dynamic programming. Similar statistical guarantees would hold in that case, with parametric rates replaced by nonparametric rates. However, our results in this section imply that in settings where the error variances are known or may be estimated accurately, the much simpler method of squared $\ell_2$-loss may be used in place of a more complicated nonparametric approach.
\end{remark}

\subsection{Squared $\ell_2$-loss}

Suppose $X$ is drawn from a linear SEM~\eqref{EqnLinStruct}, where we now use $B_0$ to denote the true autoregression matrix and $\Omega_0$ to denote the true error covariance matrix. For a fixed diagonal matrix $\Omega = \diag(\sigma_1^2, \dots, \sigma_p^2)$ and a candidate matrix $B$ with columns $\{b_j\}_{j=1}^p$, define the \emph{score of} $B$ \emph{with respect to} $\Omega$ according to
\begin{equation}
	\label{EqnScore}
	\score_\Omega(B) = \E\Big[\|\Omega^{-1/2}(I-B)^T X\|_2^2\Big] = \sum_{j=1}^p \frac{1}{\sigma_j^2} \cdot \E[(X_j - b_j^T X)^2].
\end{equation}
This is a weighted squared $\ell_2$-loss, where the prediction error for the $j^\text{th}$ coordinate is weighted by the diagonal entry $\sigma_j^2$ coming from $\Omega$, and expectations are taken with respect to the true distribution on $X$.

It is instructive to compare the score function~\eqref{EqnScore} to the usual parametric maximum likelihood when $X \sim N(0, \Sigma)$. For a distribution parametrized by the pair $(B, \Omega)$, the inverse covariance matrix is $\Theta = (I-B) \Omega^{-1} (I-B)^T$, so the expected log likelihood is
\begin{align*}
	\E_{X \sim N(0, \Sigma)}[\log p_{B, \Omega}(X)] & = -\tr[(I-B) \Omega^{-1} (I-B)^T \Sigma] + \log \det[(I-B) \Omega^{-1} (I-B)^T] \\
	& = -\tr[(I-B) \Omega^{-1} (I-B)^T \Sigma] + \log \det(\Omega^{-1}) \\
	& = - \score_\Omega(B) + \log \det(\Omega^{-1}).
\end{align*}
Hence, minimizing the score over $B$ for a fixed $\Omega$ is identical to maximizing the likelihood. For non-Gaussians, however, the convenient relationship between minimum score and maximum likelihood no longer holds.

Now let $\scriptD$ denote the class of DAGs. For $G \in \scriptD$, define the score of $G$ to be
\begin{equation}
	\label{EqnScoreDAG}
	\score_\Omega(G) \defn \min_{B \in \scriptU_G} \left\{\score_\Omega(B)\right\},
\end{equation}
where
\begin{equation*}
	\scriptU_G \defn \{B \in \real^{p \times p}: B_{jk} = 0 \text{ when } (j,k) \not\in E(G)\}
\end{equation*}
is the set of matrices that are consistent with the structure of $G$.

\begin{remark}
	\label{RemScore}
Examining the form of the score function~\eqref{EqnScore}, we see that if $\{\Pa_G(j)\}_{j=1}^p$ denotes the parent sets of nodes in $G$, then the matrix
\begin{equation*}
	B_G \defn \arg\min_{B \in \scriptU_G} \{\score_\Omega(B)\}
\end{equation*}
is unique, and the columns of $B_G$ are equal to the coefficients of the best linear predictor of $X_j$ regressed upon $X_{\Pa_G(j)}$. Furthermore, the value of $B_G$ does not depend on $\Omega$.
\end{remark}

The following lemma relates the score of the underlying DAG $G_0$ to the score of the true autoregression matrix $B_0$. In fact, the score of any DAG containing $G_0$ has the same score. The proof is contained in Appendix~\ref{AppScoreDAG}.
\begin{lemma}
	\label{LemScoreDAG}
	Suppose $X$ follows a linear SEM with autoregression matrix $B_0$, and let $G_0$ denote the underlying DAG. Consider any $G \in \scriptD$ such that $G_0 \subseteq G$. Then for any diagonal weight matrix $\Omega$, we have
\begin{equation*}
	\score_\Omega(G) = \score_\Omega(B_0),
\end{equation*}
and $B_0$ is the unique minimizer of $\score_\Omega(B)$ over $\scriptU_G$.
\end{lemma}

We now turn to the main theorem of this section, in which we consider the problem of minimizing $\score_\Omega(B)$ with respect to all matrices $B$ that are permutation similar to upper triangular matrices. Such a result is needed to validate our choice of score function, since when the DAG structure is not known a priori, the space of possible autoregression matrices must include all $\scriptU \defn \bigcup_{G \in \scriptD} \scriptU_G$. Note that $\scriptU$ may be equivalently defined as the set of all matrices that are permutation similar to upper triangular matrices. We have the following vital result:
\begin{theorem}
	\label{ThmScore}
	Given a linear SEM~\eqref{EqnLinStruct} with error covariance matrix $\alpha\Omega_0$ and autoregression matrix $B_0$, where $\alpha > 0$, we have
	\begin{equation}
		\label{EqnScoreUB}
		\score_{\alpha \Omega_0}(B) \ge \score_{\alpha \Omega_0}(B_0) = p, \qquad \forall B \in \scriptU,
	\end{equation}
with equality if and only if $B = B_0$.
\end{theorem}
The proof of Theorem~\ref{ThmScore}, which is based on matrix algebra, is contained in Section~\ref{SecThmScore}. In particular, Theorem~\ref{ThmScore} implies that the squared $\ell_2$-loss function~\eqref{EqnScore} is indeed an appropriate measure of model fit when the components are correctly weighted by the diagonal entries of $\Omega_0$.

Note, however, that Theorem~\ref{ThmScore} requires the score to be taken with respect to (a multiple of) the true error covariance matrix $\Omega_0$. The following example gives a cautionary message that if the weights $\Omega$ are chosen incorrectly, minimizing $\score_\Omega(B)$ may produce a structure that is \emph{inconsistent} with the true model.

\begin{example}
	\label{ExaWrongOmega}
	Suppose $(X_1, X_2)$ is distributed according to the following linear SEM:
	\begin{equation*}
		X_1 = \epsilon_1, \quad \text{and} \quad X_2 = -\frac{X_1}{2} + \epsilon_2,
	\end{equation*}
so the autoregression matrix is given by $B_0 = \left(\begin{matrix} 0 & -\frac{1}{2} \\ 0 & 0 \end{matrix}\right)$. Let $\Omega_0 = \left(\begin{matrix} 1 & 0 \\ 0 & \frac{1}{4} \end{matrix}\right)$. Consider $B_1 = \left(\begin{matrix} 0 & 0 \\ -1 & 0 \end{matrix}\right)$. Then
\begin{equation}
	\label{EqnLessScore}
		\score_I(B_1) < \score_I(B_0),
\end{equation}
so using squared $\ell_2$-loss weighted by the identity will select an inappropriate model.
\end{example}

\begin{proof}
To verify equation~\eqref{EqnLessScore}, we first compute
\begin{equation*}
	\Sigma = (I-B_0)^{-T} \Omega_0 (I-B_0) = \left(\begin{matrix} 1 & -\frac{1}{2} \\ -\frac{1}{2} & \frac{1}{2}\end{matrix}\right).
\end{equation*}
Then
\begin{align*}
	\E[\|X - B_1^T X\|_2^2] & = \tr\left[(I-B_1)^T \Sigma (I-B_1)\right] = \tr\left[\left(\begin{matrix} 1/2 & 0 \\ 0 & 1/2\end{matrix}\right)\right] = 1, \\
	\E[\|X - B_0^T X\|_2^2] & = \tr\left[(I-B_0)^T \Sigma (I-B_0)\right] = \tr\left[\left(\begin{matrix} 1 & 0 \\ 0 & 1/4\end{matrix}\right)\right] = \frac{5}{4},
\end{align*}
implying inequality~\eqref{EqnLessScore}.
\end{proof}

\subsection{Identifiability of linear SEMs}
\label{SecIdentify}

Theorem~\ref{ThmScore} also has a useful consequence in terms of identifiability of a linear SEM, which we state in the following corollary:
\begin{corollary}
	\label{CorIdentify}
	Consider a fixed diagonal covariance matrix $\Omega_0$, and consider the class of linear SEMs parametrized by the pair $(B, \alpha \Omega_0)$, where $B \in \scriptU$ and $\alpha > 0$ is a scale factor. Then the true model $(B_0, \alpha_0 \Omega_0)$ is identifiable. In particular, the class of homoscedastic linear SEMs is identifiable.
\end{corollary}

\begin{proof}
	By Theorem~\ref{ThmScore}, the matrix $B_0$ is the unique minimizer of $\score_{\Omega_0}(B)$. Since $\alpha_0 \cdot (\Omega_0)_{11} = \var[X_1]$, the scale factor $\alpha_0$ is also uniquely identifiable. The statement about homoscedasticity follows by taking $\Omega_0 = I$.
\end{proof}
	
Corollary~\ref{CorIdentify} should be viewed in comparison to previous results in the literature regarding identifiability of linear SEMs. Theorem 1 of~\cite{PetBuh13} states that when $X$ is Gaussian and $\epsilon$ is an i.i.d.\ Gaussian vector with $\cov[\epsilon] = \alpha \Omega_0$, the model is identifiable. Indeed, our Corollary~\ref{CorIdentify} implies that result as a special case, but it does not impose any additional conditions concerning Gaussianity. \cite{ShiEtal06} establish identifiability of a linear SEM when $\epsilon$ is a vector of independent, non-Gaussian errors, by reducing to ICA, but our result does not require errors to be non-Gaussian.

The significance of Corollary~\ref{CorIdentify} is that it supplies an elegant proof showing that the model is still identifiable even in the presence of both Gaussian and non-Gaussian components, provided the error variances are specified up to a scalar multiple. Since any multivariate Gaussian distribution may be written as a linear SEM with respect to an arbitrary ordering, some constraint such as variance scaling or non-Gaussianity is necessary in order to guarantee identifiability.

\subsection{Misspecification of variances}
\label{SecMisspec}

Theorem~\ref{ThmScore} implies that when the diagonal variances of $\Omega_0$ are known up to a scalar factor, the weighted $\ell_2$-loss~\eqref{EqnScore} may be used as a score function for linear SEMs. Example~\ref{ExaWrongOmega} shows that when $\Omega$ is misspecified, we may have $B_0 \not\in \arg\min_{B \in \scriptU} \left\{\score_\Omega(B)\right\}$. In this section, we further study the effect when $\Omega$ is misspecified. Intuitively, provided $\Omega$ is close enough to $\Omega_0$ (or a multiple thereof), minimizing $\score_\Omega(B)$ with respect to $B$ should still yield the correct $B_0$.

Consider an arbitrary diagonal weight matrix $\Omega_1$. We first provide  bounds on the ratio between entries of $\Omega_0$ and $\Omega_1$ which ensure that $B_0 = \arg\min_{B \in \scriptU} \left\{\score_{\Omega_1}(B)\right\}$, even though the model is misspecified. Let
\begin{equation*}
	a_{\max} \defn \lambda_{\max}(\Omega_0 \Omega_1^{-1}), \quad \text{and} \quad a_{\min} \defn \lambda_{\min}(\Omega_0 \Omega_1^{-1}),
\end{equation*}
denote the maximum and minimum ratios between corresponding diagonal entries of $\Omega_1$ and $\Omega_0$. Now define the additive gap between the score of $G_0$ and the next best DAG, given by
\begin{equation}
	\label{EqnXi}
	\xi \defn \min_{G \in \scriptD, \; G \not\supseteq G_0} \{ \score_{\Omega_0}(G) - \score_{\Omega_0}(G_0) \} = \min_{G \in \scriptD, G \not\supseteq G_0} \{\score_{\Omega_0}(G)\} - p.
\end{equation}
By Theorem~\ref{ThmScore}, we know that $\xi > 0$. The following theorem provides a sufficient condition for correct model selection in terms of the gap $\xi$ and the ratio $\frac{a_{\max}}{a_{\min}}$, which are both invariant to the scale factor $\alpha$. It is a measure of robustness for how roughly the entries of $\Omega_0$ may be approximated and still produce $B_0$ as the unique minimizer. The proof of the theorem is contained in Appendix~\ref{AppRatio}.
\begin{theorem}
	\label{ThmRatio}
	Suppose
	\begin{equation}
		\label{EqnARatio}
		\frac{a_{\max}}{a_{\min}} \le 1 + \frac{\xi}{p}.
	\end{equation}
	Then $B_0 \in \arg\min_{B \in \scriptU} \left\{\score_{\Omega_1}(B)\right\}$. If inequality~\eqref{EqnARatio} is strict, then $B_0$ is the unique minimizer of $\score_{\Omega_1}(B)$.
\end{theorem}

\begin{remark}
Theorem~\ref{ThmRatio} provides an error allowance concerning the accuracy to which we may specify the error covariances and still recover the correct autoregression matrix $B_0$ from an improperly weighted score function. In the case $\Omega_1 = \alpha \Omega_0$, we have $a_{\max} = a_{\min} = 1$, so the condition~\eqref{EqnARatio} is always strictly satisfied, which is consistent with our earlier result in Theorem~\ref{ThmScore} that $B_0 = \arg\min_{B \in \scriptU} \left\{\score_{\alpha \Omega_0}(B)\right\}$.
\end{remark}

Naturally, the error tolerance specified by Theorem~\ref{ThmRatio} is a function of the gap $\xi$ between the true DAG $G_0$ and the next best candidate: If $\xi$ is larger, the score is more robust to misspecification of the weights $\Omega$. Note that if we restrict our search space from the full set of DAGs $\scriptD$ to some smaller space $\scriptD'$, so $B \in \bigcup_{G \in \scriptD'} \scriptU_G$, we may restate the condition in Theorem~\ref{ThmRatio} in terms of the gap
\begin{equation}
	\label{EqnGapPrime}
	\xi(\scriptD') \defn \min_{G \in \scriptD', \; G \not\supseteq G_0} \{\score_{\Omega_0}(G) - \score_{\Omega_0}(G_0) \},
\end{equation}
which may be considerably larger than $\xi$ when $\scriptD'$ is much smaller than $\scriptD$ (cf.\ equation~\eqref{EqnXiOmega} and Section~\ref{SecGap} on weakening the gap condition below).

Specializing to the case when $\Omega_1 = I$, we may interpret Theorem~\ref{ThmRatio} as providing a window of variances around which we may treat a heteroscedastic model as homoscedastic, and use the simple (unweighted) squared $\ell_2$-score to recover the correct model. See Lemma~\ref{LemTwoVar} in the next section for a concrete example.

\subsection{Example: 2- and 3-variable models}
\label{SecSmallDAGs}

In this section, we develop examples illustrating the gap $\xi$ introduced in Section~\ref{SecMisspec}. We study two cases, involving two and three variables, respectively.

\subsubsection{Two variables}

We first consider the simplest example with a two-variable directed graph. Suppose the forward model is defined by
\begin{equation*}
	B_0 = \left(
	\begin{array}{cc}
		0 & b_0 \\
		0 & 0 \\
	\end{array}
	\right), \qquad \Omega_0 = \left(
	\begin{array}{cc}
		d_1^2 & 0 \\
		0 & d_2^2
	\end{array}
	\right),
\end{equation*}
and consider the backward matrix defined by the autoregression matrix
\begin{equation}
	\label{Eqn2True}
	B = \left(
	\begin{array}{cc}
		0 & 0 \\
		b & 0 \\
	\end{array}
	\right).
\end{equation}
The forward and backward models are illustrated in Figure~\ref{Fig2DAG}.
\begin{figure}
\begin{center}
	\begin{tabular}{cc}
		\includegraphics[width=3cm]{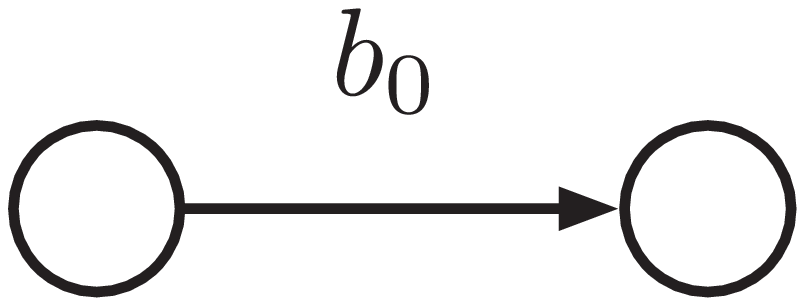} & \includegraphics[width=3cm]{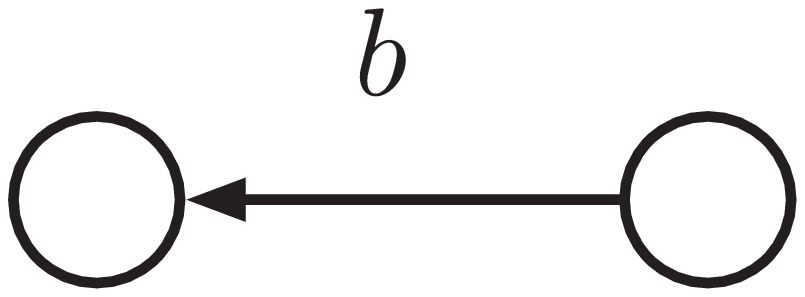} \\
		(a) & (b)
	\end{tabular}
\end{center}
\caption{Two-variable DAG. (a) The forward model. (b) The backward model.}
\label{Fig2DAG}
\end{figure}

A straightforward calculation shows that
\begin{equation*}
	\score_{\Omega_0}(B) = 2 + b^2 b_0^2 + \left(b \frac{d_2}{d_1} - b_0 \frac{d_1}{d_2}\right)^2,
\end{equation*}
which is minimized for $b = \frac{b_0}{b_0^2 + \frac{d_2^2}{d_1^2}}$, implying that
\begin{equation}
	\label{Eqn2Gap}
	\xi = \min_b \{\score_{\Omega_0}(B) - \score_{\Omega_0}(B_0)\} = \frac{b_0^4}{\frac{d_2^4}{d_1^4} + b_0^2 \frac{d_2^2}{d_1^2}}.
\end{equation}
We see that the gap $\xi$ grows with the strength of the true edge $b_0$, when $|b_0| > 1$, and is symmetric with respect to the sign of $b_0$. The gap also grows with the magnitude of the ratio $\frac{d_1}{d_2}$.

To gain intuition for the interplay between $b_0$ and $\frac{d_1}{d_2}$, we derive the following lemma, a corollary of Theorem~\ref{ThmRatio} specialized to the case of the two-variable DAG:
\begin{lemma}
	\label{LemTwoVar}
	Consider the two-variable DAG defined by equation~\eqref{Eqn2True}. Let $\Omega_1 = I_2$ and define $r \defn \frac{d_2}{d_1}$. Suppose the following conditions hold:
	\begin{equation}
		\label{EqnBCond}
		b_0^2 \ge
		\begin{cases}
			r^2\left((r^2-1) + \sqrt{r^4 - 1}\right), & \text{ if } r \ge 1, \\
			(1-r^2) + \sqrt{1 - r^4}, & \text{ if } r \le 1.
		\end{cases}
	\end{equation}
	Then $B_0 = \arg\min_{B \in \scriptU} \{\score_{\Omega_1}(B)\}$; i.e., $B_0$ is the unique minimizer of the score function under the unweighted squared-$\ell_2$ loss.
\end{lemma}
Lemma~\ref{LemTwoVar} is proved in Appendix~\ref{AppTwoVar}.

\begin{remark}
Note that the two right-hand expressions in inequality~\eqref{EqnBCond} are similar, although the expression in the case $r^2 \ge 1$ contains an extra factor of $r^2$, so the sufficient condition is stronger. Both lower bounds in equation~\eqref{EqnBCond} increase with $|r-1|$, which agrees with intuition: If the true model is more non-homoscedastic, the strength of the true edge must be stronger in order for the unweighted squared-$\ell_2$ score to correctly identify the model. When $r = 1$, we have the vacuous condition $b_0^2 \ge 0$, since $\Omega_1 = \alpha \Omega_0$ and the variances are correctly specified, so Theorem~\ref{ThmScore} implies $B_0 = \arg\min_{B \in \scriptU} \{\score_{\Omega_1}(B)\}$ for any choice of $b_0$.
\end{remark}

\subsubsection{$v$-structure}

We now examine a three-variable graph. Suppose the actual graph involves a $v$-structure, as depicted in Figure~\ref{Fig3True}, and is parametrized by the matrices
\begin{equation}
	\label{Eqn3True}
	B_0 = \left(
	\begin{array}{ccc}
		0 & 0 & b_{13} \\
		0 & 0 & b_{23} \\
		0 & 0 & 0
	\end{array}
	\right), \qquad \Omega = \left(
	\begin{array}{ccc}
		d_1^2 & 0 & 0 \\
		0 & d_2^2 & 0 \\
		0 & 0 & d_3^2
	\end{array}
	\right).
\end{equation}
\begin{figure}
	\begin{center}
		\includegraphics[width=3cm]{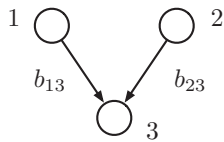}
	\end{center}
	\caption{Three-variable DAG with $v$-structure.}
	\label{Fig3True}
\end{figure}
We have the following lemma, proved in Appendix~\ref{App3DAG}:
\begin{lemma}
	\label{Lem3DAG}
	Consider the three-variable DAG characterized by Figure~\ref{Fig3True} and equations~\eqref{Eqn3True}. The gap $\xi$ defined by equation~\eqref{EqnXi} is given by
	\begin{equation*}
		\xi = \min\left\{\frac{b_{23}^4}{\frac{d_3^4}{d_2^4} + b_{23}^2 \frac{d_3^2}{d_2^2}}, \; \frac{b_{13}^4}{\frac{d_3^4}{d_1^4} + b_{13}^2 \frac{d_3^2}{d_1^2}}\right\}.
	\end{equation*}
\end{lemma}

A key reduction in the proof of Lemma~\ref{Lem3DAG} is to note that we only need to consider a relatively small number of DAGs, given by Figure~\ref{Fig3Alt} in Appendix~\ref{App3DAG}. Indeed, for $G_1 \subseteq G_2$, we have
$\score_{\Omega_0}(G_2) \le \score_{\Omega_0}(G_1)$, so it suffices to consider maximal elements in the poset of DAGs not containing the true DAG $G_0$.

\begin{remark}
\label{RemGap}
Note that the form of the gap in Lemma~\ref{Lem3DAG} is very similar to the form for the two-variable model, and the individual ratios scale with the strength of the edge and the ratio of the corresponding error variances. Indeed, we could derive a version of Lemma~\ref{LemTwoVar} for the three-variable model, giving lower bounds on the edge strengths $b_{23}^2$ and $b_{13}^2$ that guarantee the accuracy of the unweighted squared $\ell_2$-loss; however, the conditions would be more complicated. It is interesting to note from our calculations in Appendix~\ref{App3DAG} that the gap between models accumulates according to the number of edge reversals from the misspecified model: Reversing the directions of edges $(2,3)$ and $(1,3)$ in succession leads to an additional term in the expressions for $\xi_1$ and $\xi_2$ in equations~\eqref{EqnXi1} and~\eqref{EqnXi2} below. We will revisit these observations in Section~\ref{SecGap}, where we define a version of the gap function rescaled by the number of nodes that differ in their parents sets.
\end{remark}

\subsection{Proof of Theorem~\ref{ThmScore}}
\label{SecThmScore}

First note that for a constant $\alpha > 0$, we have
\begin{equation*}
	\score_{\alpha\Omega}(B) = \frac{1}{\alpha} \cdot \score_\Omega(B),
\end{equation*}
so minimizing $\score_{\alpha\Omega}(B)$ is equivalent to minimizing $\score_\Omega(B)$. Hence, it is sufficient to prove the statement for $\alpha = 1$.

Recalling equation~\eqref{EqnCovs}, we may write
\begin{align*}
	\score_{\Omega_0}(B) & = \E_{B_0}[\|\Omega_0^{-1/2}(I-B)^T X\|_2^2] \\
	& = \tr\left[\Omega_0^{-1/2}(I-B)^T \cdot \cov_{B_0}[X] \cdot (I-B) \Omega_0^{-1/2}\right] \\
	& = \tr\left[\Omega_0^{-1/2}(I-B)^T \cdot (I-B_0)^{-T} \Omega_0 (I-B_0)^{-1} \cdot (I-B) \Omega_0^{-1/2}\right].
\end{align*}
Now note that
\begin{align*}
	(I-B) \Omega_0^{-1/2} & = \Omega_0^{-1/2} (I - \Omega_0^{1/2} B \Omega_0^{-1/2}) \defn \Omega_0^{-1/2} (I-\Btil), \\
	\Omega_0^{1/2} (I-B_0)^{-1} & = (I - \Omega_0^{1/2} B_0 \Omega_0^{-1/2})^{-1}\Omega_0^{1/2} \defn (I-\Btil_0)^{-1} \Omega_0^{1/2},
\end{align*}
where $\Btil, \Btil_0 \in \scriptU$. Hence, we may rewrite
\begin{align*}
	\score_{\Omega_0}(B) & = \tr\left[(I-\Btil)^T \Omega_0^{-1/2} \cdot \Omega_0^{1/2} (I-\Btil_0)^{-T} \cdot (I-\Btil_0)^{-1} \Omega_0^{1/2} \cdot \Omega_0^{-1/2} (I-\Btil)\right] \\
	& = \tr\left[(I-\Btil)^T (I-\Btil_0)^{-T} (I-\Btil_0)^{-1} (I-\Btil)\right] \\
	& = \tr\left[(I-\Btil)(I-\Btil)^T (I-\Btil_0)^{-T} (I-\Btil_0)^{-1}\right].
\end{align*}
Since $\Btil, \Btil_0 \in \scriptU$, the matrices $I-\Btil$ and $I-\Btil_0$ are both permutation similar to lower triangular matrices with 1's on the diagonal. Hence, Lemma~\ref{LemMinTrace} in Appendix~\ref{AppMatrix} implies
\begin{equation*}
	\score_{\Omega_0}(B) \ge p,
\end{equation*}
with equality if and only if
\begin{equation*}
	I - \Btil = I - \Btil_0,
\end{equation*}
or equivalently, $B = B_0$, as claimed.

\section{Consequences for statistical estimation}
\label{SecStat}

The population-level results in Theorems~\ref{ThmInvDAG} and~\ref{ThmScore} provide a natural avenue for estimating the DAG of a linear SEM from data. In this section, we outline how the true DAG may be estimated in the presence of fully-observed or systematically corrupted data. Our method is applicable also in the high-dimensional setting, assuming the moralized DAG is sufficiently sparse.

Our inference algorithm consists of two main components:
\begin{enumerate}
	\item Estimate the moralized DAG $\scriptM(G_0)$ using the inverse covariance matrix of $X$.
	\item Search through the space of DAGs consistent with $\scriptM(G_0)$, and find the DAG that minimizes $\score_\Omega(B)$.
\end{enumerate}

Theorem~\ref{ThmInvDAG} and Assumption~\ref{AsInvDAG} ensure that for almost every choice of autoregression matrix $B_0$, the support of the true inverse covariance matrix $\Theta_0$ exactly corresponds to the edge set of the moralized graph. Theorem~\ref{ThmScore} ensures that when the weight matrix $\Omega$ is chosen appropriately, $B_0$ will be the unique minimizer of $\score_\Omega(B)$.

\subsection{Fully-observed data}

We now present concrete statistical guarantees for the correctness of our algorithm in the usual setting when $\{x_i\}_{i=1}^n$ are fully-observed and i.i.d. Recall that a random variable $X$ is \emph{sub-Gaussian} with parameter $\sigma^2$ if
\begin{equation*}
	\E[\exp(\lambda(X - \E[X]))] \le \exp\left(\frac{\sigma^2 \lambda^2}{2}\right), \qquad \forall \lambda \in \real.
\end{equation*}
If $X \in \real^p$ is a random vector, it is sub-Gaussian with parameter $\sigma^2$ if $v^T X$ is a sub-Gaussian random variable with parameter $\sigma^2$ for all unit vectors $v \in \real^p$.

\subsubsection{Estimating the inverse covariance matrix}
\label{SecEstInverse}

We first consider the problem of inferring $\Theta_0$. Let
\begin{equation*}
	\Thetamin \defn \min_{j,k} \big\{|(\Theta_0)_{jk}|: (\Theta_0)_{jk} \neq 0\big\}
\end{equation*}
denote the magnitude of the minimum nonzero element of $\Theta_0$. We consider the following two scenarios:

\paragraph{Low-dimensional setting.} If $n \ge p$, the sample covariance matrix $\Sigmahat = \frac{1}{n} \sum_{i=1}^n x_i x_i^T$ is invertible, and we use the estimator
\begin{equation*}
	\Thetahat = (\Sigmahat)^{-1}.
\end{equation*}
We have the following lemma, which follows from standard bounds in random matrix theory:

\begin{lemma}
	\label{LemLowDim}
	Suppose the $x_i$'s are i.i.d.\ sub-Gaussian vectors with parameter $\sigma^2$. With probability at least $1 - c_1 \exp(-c_2 p)$, we have
	\begin{equation*}
		\|\Thetahat - \Theta_0\|_{\max} \le c_0 \sigma^2 \sqrt{\frac{p}{n}},
	\end{equation*}
	and thresholding $\Thetahat$ at level $\tau = c_0 \sigma^2 \sqrt{\frac{p}{n}}$ succeeds in recovering $\supp(\Theta_0)$, if $\Thetamin > 2\tau$.
\end{lemma}
For the proof, see Appendix~\ref{AppLowDim}. Here, we use to $\|\cdot\|_{\max}$ denote the elementwise $\ell_\infty$-norm of a matrix.

\paragraph{High-dimensional setting.} If $p > n$, we assume each row of the true inverse covariance matrix $\Theta_0$ is $d$-sparse. Then we use the graphical Lasso:
\begin{equation}
	\label{EqnGraphLasso}
	\Thetahat \in \arg\min_{\Theta \succeq 0} \left\{\tr(\Theta \Sigmahat) - \log\det(\Theta) + \lambda \sum_{j \neq k} |\Theta_{jk}|\right\}.
\end{equation}
Standard results~\citep{RavEtal11} establish the statistical consistency of the graphical Lasso~\eqref{EqnGraphLasso} as an estimator for the inverse covariance matrix in the setting of sub-Gaussian observations; consequently, we omit the proof of the following lemma.
\begin{lemma}
	\label{LemGLasso}
	Suppose the $x_i$'s are i.i.d.\ sub-Gaussian vectors with parameter $\sigma^2$. Suppose the sample size satisfies $n \ge Cd \log p$. With probability at least $1 - c_1 \exp(-c_2 \log p)$, we have
	\begin{equation*}
		\|\Thetahat - \Theta_0\|_{\max} \le c_0 \sigma^2 \sqrt{\frac{\log p}{n}},
	\end{equation*}
	and thresholding $\Thetahat$ at level $\tau = c_0 \sigma^2 \sqrt{\frac{\log p}{n}}$ succeeds in recovering $\supp(\Theta_0)$, if $\Thetamin > 2 \tau$.
\end{lemma}
Alternatively, we may perform nodewise regression with the ordinary Lasso~\citep{MeiBuh06} to recover the support of $\Theta_0$, with similar rates for statistical consistency.

\subsubsection{Scoring candidate DAGs}

Moving on to the second step of the algorithm, we need to estimate the score functions $\score_\Omega(B)$ of candidate DAGs and choose the minimally scoring candidate. In this section, we focus on methods for estimating an empirical version of the score function and derive rates for statistical estimation under certain models. If the space of candidate DAGs is sufficiently small, we may evaluate the empirical score function for every candidate DAG and select the optimum. In Section~\ref{SecComputation}, we describe computationally efficient procedures based on dynamic programming to choose the optimal DAG when the candidate space is too large for naive search.

The input of our algorithm is the sparsely estimated inverse covariance matrix $\Thetahat$ from Section~\ref{SecEstInverse}. For a matrix $\Theta$, define the candidate neighborhood sets
\begin{equation*}
	N_\Theta(j) \defn \{k: k \neq j \text{ and } \Theta_{jk} \neq 0\}, \qquad \forall j,
\end{equation*}
and let
\begin{equation*}
	\scriptD_\Theta \defn \{G \in \scriptD: \Pa_G(j) \subseteq N_\Theta(j), \quad \forall j\}
\end{equation*}
denote the set of DAGs with skeleton contained in the graph defined by $\supp(\Theta)$. By Theorem~\ref{ThmInvDAG} and Assumption~\ref{AsInvDAG}, we have $G_0 \in \scriptD_{\Theta_0}$, so if $\supp(\Thetahat) \supseteq \supp(\Theta_0)$, which occurs with high probability under the conditions of Section~\ref{SecEstInverse}, it suffices to search over the reduced DAG space $\scriptD_{\Thetahat}$.

\begin{remark}
	In fact, we could reduce the search space even further to only include DAGs with moralized graph equal to the undirected graph defined by $\supp(\Theta)$. The dynamic programming algorithm to be described in Section~\ref{SecComputation} only requires as input a superset of the skeleton; for alternative versions of the dynamic programming algorithm taking as input a superset of the moralized graph, we would indeed restrict $\scriptD_\Theta$ to DAGs with the correct moral structure.
\end{remark}

We now consider an arbitrary $d$-sparse matrix $\Theta$, with $d \le n$, and take $G \in \scriptD_{\Theta}$. By Remark~\ref{RemScore}, we have
\begin{equation}
	\label{EqnDecompose}
	\score_\Omega(G) = \sum_{j=1}^p f_{\sigma_j}(\Pa_G(j)),
\end{equation}
where
\begin{equation*}
	f_{\sigma_j}(S) \defn \frac{1}{\sigma_j^2} \cdot \E[(x_j - b_j^T x_S)^2],
\end{equation*}
and $b_j^T x_S$ is the best linear predictor for $x_j$ regressed upon $x_S$. In order to estimate $\score_\Omega(G)$, we use the empirical functions
\begin{equation}
	\label{EqnFhatClean}
	\fhat_{\sigma_j}(S) \defn \frac{1}{\sigma_j^2} \cdot \frac{1}{n} \sum_{i=1}^n (x_{ij} - x_{i,S}^T \bhat_j)^2 = \frac{1}{\sigma_j^2} \cdot \frac{1}{n} \|X_j - X_S \bhat_j\|_2^2,
\end{equation}
where
\begin{equation*}
	\bhat_j \defn (X_S^T X_S)^{-1} X_S^T X_j
\end{equation*}
is the usual ordinary least squares solution for linear regression of $X_j$ upon $X_S$. We will take $S \subseteq N_{\Theta}(j)$, so since $|N_{\Theta}(j)| \le d \le n$, the matrix $X_S^TX_S$ is invertible w.h.p. The following lemma, proved in Appendix~\ref{AppScoreConc}, provides rates of convergence for the empirical score function:
\begin{lemma}
	\label{LemScoreConc}
	Suppose the $x_i$'s are i.i.d.\ sub-Gaussian vectors with parameter $\sigma^2$. Suppose $d \le n$ is a parameter such that $|N_{\Theta}(j)| \le d$ for all $j$. Then
	\begin{equation}
		\label{EqnScoreConc}
		|\fhat_{\sigma_j}(S) - f_{\sigma_j}(S)| \le \frac{c_0 \sigma^2}{\sigma_j^2} \sqrt{\frac{\log p}{n}}, \qquad \forall j \text{ and } S \subseteq N_{\Theta}(j),
	\end{equation}
	with probability at least $1 - c_1 \exp(-c_2 \log p)$.
\end{lemma}
In particular, we have the following result, proved in Appendix~\ref{AppPGap}, which provides a sufficient condition for the empirical score functions to succeed in selecting the true DAG. Here,
\begin{equation}
	\label{EqnXiOmega}
	\xi_\Omega(\scriptD_\Theta) \defn \min_{G \in \scriptD_\Theta, G \not\supseteq G_0} \left\{\score_\Omega(G) - \score_\Omega(G_0)\right\}
\end{equation}
is the gap between $G_0$ and the next best DAG in $\scriptD_\Theta$. Note that the expression~\eqref{EqnXiOmega} is reminiscent of the expression~\eqref{EqnGapPrime} defined in Section~\ref{SecMisspec}, but we now allow $\Omega$ to be arbitrary.
\begin{lemma}
	\label{LemPGap}
	Suppose inequality~\eqref{EqnScoreConc} holds, and suppose
\begin{equation}
	\label{EqnPGap}
	c_0 \sigma^2 \sqrt{\frac{\log p}{n}} \cdot \sum_{j=1}^p \frac{1}{\sigma_j^2} < \frac{\xi_\Omega(\scriptD_\Theta)}{2}.
\end{equation}
Then
\begin{equation}
	\label{EqnGMin}
	\scorehat_\Omega(G_0) < \scorehat_\Omega(G), \qquad \forall G \in \scriptD_\Theta: G \not\supseteq G_0.
\end{equation}
\end{lemma}

\begin{remark}
Lemma~\ref{LemPGap} does not explicitly assume that $\Omega$ is equal to $\Omega_0$, the true matrix of error variances. However, inequality~\eqref{EqnPGap} can only be satisfied when $\xi_\Omega(\scriptD_G) > 0$; hence, $\Omega$ should be chosen such that $G_0 = \arg\min_{G \in \scriptD_\Theta, G \not\supseteq G_0} \left\{\score_\Omega(G)\right\}$. As discussed in Section~\ref{SecMisspec}, this condition holds for a wider range of choices for $\Omega$.
\end{remark}

Note that the conclusion~\eqref{EqnGMin} in Lemma~\ref{LemPGap} is not quite the same as the condition
\begin{equation}
	\label{EqnG0Min}
	G_0 = \arg\min_{G \in \scriptD_\Theta, G \not\supseteq G_0} \left\{\scorehat_\Omega(G)\right\},
\end{equation}
which is what we would need for exact recovery of our score-minimizing algorithm. The issue is that $\score_\Omega(G)$ is equal for all $G \supseteq G_0$; however the empirical scores $\scorehat_\Omega(G)$ may differ among this class, so equation~\eqref{EqnG0Min} may not be satisfied. However, it is easily seen from the proof of Lemma~\ref{LemPGap} that in fact,
\begin{equation}
	\label{EqnGSet}
	\arg\min_{G \in \scriptD_\Theta} \left\{\scorehat_\Omega(G)\right\} \subseteq \{G \in \scriptD_\Theta: G \supseteq G_0\}.
\end{equation}
By applying a thresholding procedure to the empirical score minimizer $\Ghat \supseteq G_0$ selected by our algorithm, we could then recover the true $G_0$. In other words, since $\Pa_{G_0}(j) \subseteq \Pa_{\Ghat}(j)$ for each $j$, we could use standard sparse regression techniques to recover the parent set of each node in the true DAG.

To gain some intuition for the condition~\eqref{EqnPGap}, consider the case when $\sigma_j^2 = 1$ for all $j$. Then the condition becomes
	\begin{equation}
		\label{EqnPScaling}
		c_0 \sigma^2 \sqrt{\frac{\log p}{n}} < \frac{\xi(\scriptD_\Theta)}{2p}.
	\end{equation}
If $\xi(\scriptD_\Theta) = \Omega(1)$, which might be expected based on our calculations in Section~\ref{SecSmallDAGs}, we require $n \ge Cp^2 \log p$ in order to guarantee statistical consistency, which is not a truly high-dimensional result. On the other hand, if $\xi(\scriptD_\Theta) = \Omega(p)$, as is assumed in similar work on score-based DAG learning~\citep{vanBuh13, BuhEtal13}, our method is consistent provided $\frac{\log p}{n} \rightarrow 0$. In Section~\ref{SecGap}, we relax the condition~\eqref{EqnPScaling} to a slightly weaker condition that is more likely to hold in settings of interest.

\subsection{Weakening the gap condition}
\label{SecGap}

Motivated by our comments from the previous section, we establish a sufficient condition for statistical consistency that is slightly weaker than the condition~\eqref{EqnPGap}, which still guarantees that equation~\eqref{EqnGSet} holds.

For two DAGs $G, G' \in \scriptD$, define
\begin{equation*}
	H(G, G') \defn \{j: \Pa_G(j) \neq \Pa_{G'}(j)\}
\end{equation*}
to be the set of nodes on which the parent sets differ between graphs $G$ and $G'$, and define the ratio
\begin{equation*}
	\gamma_\Omega(G, G') \defn \frac{\score_{\Omega}(G) - \score_\Omega(G')}{|H(G, G')|},
\end{equation*}
a rescaled version of the gap between the score functions. Consider the following condition:
\begin{assumption}
	\label{AsXiPrime}
	There exists $\xi' > 0$ such that
	\begin{equation}
		\label{EqnXiPrime}
		\gamma_\Omega(G_0) \defn \min_{G \in \scriptD_\Theta, G \not\supseteq G_0} \left\{ \max_{G_1 \supseteq G_0} \left\{\gamma_\Omega(G, G_1)\right\} \right\} \ge \xi'.
	\end{equation}
\end{assumption}
Note that in addition to minimizing over DAGs in the class $\scriptD_\Theta$, the expression~\eqref{EqnXiPrime} defined in Assumption~\ref{AsXiPrime} takes an inner maximization over DAGs containing $G_0$. As established in Lemma~\ref{LemScoreDAG}, we have $\score_\Omega(G_1) = \score_\Omega(G_0)$ whenever $G_0 \subseteq G_1$. However, $|H(G,G_1)|$ may be appreciably different from $|H(G,G_0)|$, and we are only interested in computing the gap ratio between a DAG $G \not\supseteq G_0$ and the closest DAG containing $G_0$.

We then have the following result, proved in Appendix~\ref{AppXiPrime}:
\begin{lemma}
	\label{LemXiPrime}
	Under Assumption~\ref{AsXiPrime}, suppose
	\begin{equation}
		\label{EqnXiHalf}
		|\fhat_{\sigma_j}(S) - f_{\sigma_j}(S)| \le \frac{\xi'}{2}, \qquad \forall j \text{ and } S \subseteq N_\Theta(j).
	\end{equation}
	Then the containment~\eqref{EqnGSet} holds.
\end{lemma}
Combining with Lemma~\ref{LemScoreConc}, we have the following corollary:
\begin{corollary}
	\label{CorScoreConc}
	Suppose the $x_i$'s are i.i.d.\ sub-Gaussian with parameter $\sigma^2$, and $|N_\Theta(j)| \le d$ for all $j$. Also suppose Assumption~\ref{AsXiPrime} holds. Then with probability at least $1 - c_1 \exp(-c_2 \log p)$, condition~\eqref{EqnGSet} is satisfied.
\end{corollary}

We now turn to the question of what values of $\xi'$ might be expected to give condition~\eqref{EqnXiPrime} for various DAGs. Certainly, we have
\begin{equation*}
	\gamma_\Omega(G,G') \ge \frac{\score_\Omega(G) - \score_\Omega(G')}{p},
\end{equation*}
so the condition holds when
\begin{equation*}
	p \cdot \xi' < \xi(\scriptD_\Theta).
\end{equation*}
However, for $\xi' = \order(\xi(\scriptD_\Theta)/p)$, Corollary~\ref{CorScoreConc} yields a scaling condition similar to inequality~\eqref{EqnPScaling}, which we wish to avoid. As motivated by our computations of the score functions for small DAGs (see Remark~\ref{RemGap} in Section~\ref{SecSmallDAGs}), the difference $\left\{\score_\Omega(G) - \score_\Omega(G_0)\right\}$ seems to increase linearly with the number of edge reversals needed to transform $G_0$ to $G$. Hence, we might expect $\gamma_\Omega(G,G_0)$ to remain roughly constant, rather than decreasing linearly with $p$. The following lemma verifies this intuition in a special case. For a review of junction tree terminology, see Appendix~\ref{AppJT}.
\begin{lemma}
	\label{LemSingleton}
	Suppose the moralized graph $\scriptM(G_0)$ admits a junction tree representation with only singleton separator sets. Let $C_1, \dots, C_k$ denote the maximal cliques in $\scriptM(G_0)$, and let $\{G^\ell_0\}_{\ell=1}^k$ denote the corresponding restrictions of $G_0$ to the cliques. Then
	\begin{equation*}
		\gamma_\Omega(G_0) \ge \min_{1 \le \ell \le k} \gamma_\Omega(G^\ell_0),
	\end{equation*}
	where
	\begin{equation*}
		\gamma_\Omega(G^\ell_0) \defn \min_{G^\ell \in \scriptD_\Theta \mid_{C_\ell}, G^\ell \not\supseteq G^\ell_0} \left\{ \max_{G^\ell_1 \supseteq G^\ell_0} \left\{\frac{\score_\Omega(G^\ell) - \score_\Omega(G^\ell_1)}{|H(G^\ell, G^\ell_1)|}\right\} \right\}
	\end{equation*}	
	is the gap ratio computed over DAGs restricted to clique $C_\ell$ that are consistent with the moralized graph.
\end{lemma}
The proof is contained in Appendix~\ref{AppSingleton}.

We might expect the gap ratio $\gamma_\Omega(G^\ell_0)$ to be a function of the size of the clique. In particular, if the treewidth of $\scriptM(G_0)$ is bounded by $w$ and we have $\gamma_\Omega(G^\ell_0) \ge \xi_w$ for all $\ell$, Lemma~\ref{LemSingleton} implies that
\begin{equation*}
	\gamma_\Omega(G_0) \ge \xi_w,
\end{equation*}
and we only need the parameter $\xi'$ appearing in Assumption~\ref{AsXiPrime} to be larger than $\xi_w$, rather than scaling as the inverse of $p$. We expect a version of Lemma~\ref{LemSingleton} to hold for graphs with bounded treewidth even when the separator sets have larger cardinality, but a full generalization of Lemma~\ref{LemSingleton} and a more accurate characterization of $\gamma_\Omega(G_0)$ for arbitrary graphs is beyond the scope of this paper.

\subsection{Systematically corrupted data}

We now describe how our algorithm for DAG structure estimation in linear SEMs may be extended easily to accommodate systematically corrupted data. This refers to the setting where we observe noisy surrogates $\{z_i\}_{i=1}^n$ in place of $\{x_i\}_{i=1}^n$. Two common examples include the following:
\begin{itemize}
	\item[(a)] \emph{Additive noise.} We have $z_i = x_i + w_i$, where $w_i \condind x_i$ is additive noise with known covariance $\Sigma_w$.
	\item[(b)] \emph{Missing data.} This is one instance of the more general setting of multiplicative noise. For each $1 \le j \le p$, and independently over coordinates, we have
	\begin{equation*}
		z_{ij} =
		\begin{cases} x_{ij}, & \text{with probability} \quad 1-\alpha, \\
		\star, & \text{with probability} \quad \alpha,
		\end{cases}
	\end{equation*}
	where the missing data probability $\alpha$ is either estimated or known.
\end{itemize}

We again divide our discussion into two parts: estimating $\Theta_0$ and computing score functions based on corrupted data.

\subsubsection{Inverse covariance estimation}

Following the observation of~\cite{LohWai12}, the graphical Lasso~\eqref{EqnGraphLasso} may still be used to estimate the inverse covariance matrix $\Theta_0$ in the high-dimensional setting, where we plug in a suitable estimator $\Gamhat$ for the covariance matrix $\Sigma = \cov(x_i)$, based on the corrupted observations $\{z_i\}_{i=1}^n$. For instance, in the additive noise scenario, we may take
\begin{equation}
	\label{EqnAddGamma}
	\Gamhat = \frac{Z^TZ}{n} - \Sigma_w,
\end{equation}
and in the missing data setting, we may take
\begin{equation}
	\label{EqnMissGamma}
	\Gamhat = \frac{Z^TZ}{n} \odot M,
\end{equation}
where $\odot$ denotes the Hadamard product and $M$ is the matrix with diagonal entries equal to $\frac{1}{1-\alpha}$ and off-diagonal entries equal to $\frac{1}{(1-\alpha)^2}$.

Assuming conditions such as sub-Gaussianity, the output $\Thetahat$ of the modified graphical Lasso~\eqref{EqnGraphLasso} is statistically consistent under similar scaling as in the uncorrupted setting~\citep{LohWai12}. For instance, in the additive noise setting, where the $z_i$'s are sub-Gaussian with parameter $\sigma_z^2$, Lemma~\ref{LemGLasso} holds with $\sigma^2$ replaced by $\sigma_z^2$. Analogous results hold in the low-dimensional setting, when the expressions for $\Gamhat$ in equations~\eqref{EqnAddGamma} and~\eqref{EqnMissGamma} are invertible with high probability, and we may simply use $\Thetahat = (\Gamhat)^{-1}$.

\subsubsection{Computing DAG scores}

We now describe how to estimate score functions for DAGs based on corrupted data. By equation~\eqref{EqnDecompose}, this reduces to estimating
\begin{equation*}
	f_{\sigma_j}(S) = \frac{1}{\sigma_j^2} \cdot \E[(x_j - b_j^T x_S)^2],
\end{equation*}
for a subset $S \subseteq \{1, \dots, p\} \backslash \{j\}$, with $|S| \le n$. Note that
\begin{equation*}
	\sigma_j^2 \cdot f_{\sigma_j}(S) = \Sigma_{jj} - 2b_j^T \Sigma_{S,j} + b_j^T \Sigma_{SS} b_j = \Sigma_{jj} - \Sigma_{j,S} \Sigma_{SS}^{-1} \Sigma_{S,j},
\end{equation*}
since $b_j = \Sigma_{SS}^{-1} \Sigma_{S,j}$.

Let $\Gamhat$ be the estimator for $\Sigma$ based on corrupted data used in the graphical Lasso (e.g., equations~\eqref{EqnAddGamma} and~\eqref{EqnMissGamma}). We then use the estimator
\begin{equation}
	\label{EqnFhatGen}
	\ftil_{\sigma_j}(S) = \frac{1}{\sigma_j^2} \cdot \left(\Gamhat_{jj} - \Gamhat_{j,S} \Gamhat_{SS}^{-1} \Gamhat_{S,j}\right).
\end{equation}
Note in particular that equation~\eqref{EqnFhatGen} reduces to expression~\eqref{EqnFhatClean} in the fully-observed setting. We establish consistency of the estimator in equation~\eqref{EqnFhatGen}, under the following deviation condition on $\Gamhat$:
\begin{equation}
	\label{EqnDev}
	\mprob\left(\opnorm{\Gamhat_{SS} - \Sigma_{SS}}_2 \ge \sigma^2 \left(\sqrt{\frac{d}{n}} + t\right)\right) \le c_1 \exp(-c_2 nt^2), \quad \text{ for any } |S| \le d.
\end{equation}
For instance, such a condition holds in the case of the sub-Gaussian additive noise model (cf.\ Lemma~\ref{LemSubGSpec} in Appendix~\ref{AppConcentrate}), with $\Gamhat$ given by equation~\eqref{EqnAddGamma}, where $\sigma^2 = \sigma_z^2$.

We have the following result, an extension of Lemma~\ref{LemScoreConc} applicable also for corrupted variables:
\begin{lemma}
	\label{LemNoisyScoreConc}
	Suppose $\Gamhat$ satisfies the deviation condition~\eqref{EqnDev}. Suppose $|N_\Theta(j)| \le d$ for all $j$. Then
	\begin{equation*}
		|\ftil_{\sigma_j}(S) - f_{\sigma_j}(S)| \le \frac{c_0 \sigma^2}{\sigma_j^2} \sqrt{\frac{\log p}{n}}, \qquad \forall j \text{ and } S \subseteq N_\Theta(j),
	\end{equation*}
	with probability at least $1 - c_1 \exp(-c_2 \log p)$.
\end{lemma}
The proof is contained in Appendix~\ref{AppNoisyScoreConc}. In particular, Corollary~\ref{CorScoreConc}, providing guarantees for statistical consistency, also holds.

\section{Computational considerations}
\label{SecComputation}

In practice, the main computational bottleneck in inferring the DAG structure comes from having to compute score functions over a large number of DAGs. The simplest approach of searching over all possible permutation orderings of indices gives rise to $p!$ candidate DAGs, which scales exponentially with $p$. In this section, we describe how the result of Theorem~\ref{ThmInvDAG} provides a general framework for achieving vast computational savings for finding the best-scoring DAG when data are generated from a linear SEM. We begin by reviewing existing methods, and describe how our results may be used in conjunction with dynamic programming to produce accurate and efficient DAG learning.

\subsection{Decomposable score functions}

Following the literature, we call a score function over DAGs \emph{decomposable} if it may be written as a sum of score functions over individual nodes, each of which is a function of only the node and its parents:
\begin{equation*}
%	\label{EqnDecomp}
	\score(G) = \sum_{j=1}^p \score_j(\Pa_G(j)).
\end{equation*}
Note that we allow the score functions to differ across nodes. Consistent with our earlier notation, the goal is to find the DAG $G \in \scriptD$ that minimizes $\score(G)$.

Some common examples of decomposable scores that are used for DAG inference include maximum likelihood, BDe, BIC, and AIC~\citep{Chi95}. By equation~\eqref{EqnDecompose}, the squared $\ell_2$-score is clearly decomposable, and it gives an example where $\score_j$ differs over nodes. Another interesting example is the nonparametric maximum likelihood, which extends the ordinary likelihood method for scoring DAGs~\citep{NowBuh13}.

Various recent results have focused on methods for optimizing a decomposable score function over the space of candidate DAGs in an efficient manner. Some methods include exhaustive search~\citep{SilMyl06}, greedy methods~\citep{Chi02}, and dynamic programming~\citep{OrdSze12, KorPar13}. We will focus here on a dynamic programming method that takes as input an undirected graph and outputs the best-scoring DAG with skeleton contained in the input graph.

\subsection{Dynamic programming}
\label{SecDynamic}

In this section, we detail a method due to~\cite{OrdSze12} that will be useful for our purposes. Given an input undirected graph $G_I$ and a decomposable score function, the dynamic programming algorithm finds a DAG with minimal score that has skeleton contained in $G_I$. Let $\{N_I(j)\}_{j=1}^p$ denote the neighborhood sets of $G_I$. The runtime of the dynamic programming algorithm is exponential in the treewidth $w$ of $G_I$; hence, the algorithm is only tractable for bounded-treewidth graphs.

The main steps of the dynamic programming algorithm are as follows. For a review of basic terminology of graph theory, including treewidth and tree decompositions, see Appendix~\ref{AppGraph}; for further details and a proof of correctness, see~\cite{OrdSze12}.
\begin{enumerate}
	\item Construct a \emph{tree decomposition} of $G_I$ with minimal treewidth.
	\item Construct a \emph{nice tree decomposition} of the graph. Let $\chi(t)$ denote the subset of $\{1, \dots, p\}$ associated to a node $t$ in the nice tree decomposition.
	\item Starting from the leaves of the nice tree decomposition up to the root, compute the \emph{record} for each node $t$. The record $\scriptR(t)$ is the set of tuples $(a,p,s)$ corresponding to minimal-scoring DAGs defined on the vertices $\chi^*(t)$ in the subtree attached to $t$, with skeleton contained in $G_I$. For each such DAG, $s$ is the score, $a$ lists the parent sets of vertices in $\chi(t)$, such that $a(v) \subseteq N_I(v)$ for each $v \in \chi(t)$, and $a(v)$ restricted to $\chi^*(t)$ agrees with the partial DAG; and $p$ lists the directed paths between vertices in $\chi(t)$. The records $\scriptR(t)$ may computed recursively over the nice tree decomposition as follows:
	\begin{itemize}
		\item \textbf{Join node:} Suppose $t$ has children $t_1$ and $t_2$. Then $\scriptR(t)$ is the union of tuples $(a,p,s)$ formed by tuples $(a_1, p_1, s_1) \in \scriptR(t_1)$ and $(a_2, p_2, s_2) \in \scriptR(t_2)$, where (1) $a = a_1 = a_2$; (2) $p$ is the transitive closure of $p_1 \cup p_2$; (3) $p$ contains no cycles; and (4) $s = s_1 + s_2$.
		\item \textbf{Introduce node:} Suppose $t$ is an introduce node with child $t'$, such that $\chi(t) = \chi(t') \cup \{v_0\}$. Then $\scriptR(t)$ is the set of tuples $(a,p,s)$ formed by pairs $P \subseteq N_I(v_0)$ and $(a',p',s') \in \scriptR(t')$, such that (1) $a(v_0) = P$; (2) for every $v \in \chi(t')$, we have $a(v) = a'(v)$; (3) $p$ is the transitive closure of $p' \cup \{(u,v_0): u \in P\} \cup \{(v_0, u): v_0 \in a'(u), u \in \chi(t')\}$; (4) $p$ contains no cycles; and (5) $s = s'$.
		\item \textbf{Forget node:} Suppose $t$ is a forget node with child $t'$, such that $\chi(t') = \chi(t) \cup \{v_0\}$. Then $\scriptR(t)$ is the set of tuples $(a,p,s)$ formed from tuples $(a',p',s') \in \scriptR(t')$, such that (1) $a(u) = a'(u), \; \forall u \in \chi(t)$; (2) $p = \{(u,v) \in p': u,v \in \chi(t)\}$; and (3) $s = s' + \score_{v_0}(a'(v_0))$.
	\end{itemize}
\end{enumerate}

Note that~\cite{KorPar13} present a variant of this dynamic programming method, also using a nice tree decomposition, which is applicable even for graphs with unbounded degree but bounded treewidth. They assume that the starting undirected graph $G_I$ is a superset of the moralized DAG. Their algorithm runs in time linear in $p$ and exponential in $w$. From a theoretical perspective, we are agnostic to both methods, since our results on statistical consistency of the graphical Lasso require the true moralized graph to have bounded degree. However, since $\supp(\Theta_0)$ exactly corresponds to the edge set of $\scriptM(G_0)$, the alternative method will also lead to correct recovery. In practice, the relative efficiency of the two dynamic programming algorithms will rely heavily on the structure of $\scriptM(G_0)$.

\subsection{Runtime}

We first review the runtime of various components of the dynamic programming algorithm described in Section~\ref{SecDynamic}. This is mentioned briefly in~\cite{KorPar13}, but we include details here for completeness. In our calculations, we assume the treewidth $w$ of $G$ is bounded and treat $w$ as a constant.

The first step involves constructing a tree decomposition of minimal treewidth $w$, which may be done in time $\order(p)$. The second step involves constructing a nice tree decomposition. Given a tree decomposition of width $w$, a nice tree decomposition with $\order(p)$ nodes and treewidth $w$ may be constructed in $\order(p)$ time (see Appendix~\ref{AppTree}). Finally, the third step involves computing records for nodes in the nice tree decomposition. We consider the three different types of nodes in succession. Note that
\begin{equation}
	\label{EqnRecord}
|\scriptR(t)| \le 2^{(w+1)(w+d)}, \qquad \forall t,
\end{equation}
where $d = \max_j |N_I(j)|$. This is because the number of choices of parent sets of any vertex in $\chi(t)$ is bounded by $2^d$, leading to a factor of $2^{d(w+1)}$, and the number of possible pairs that are connected by a path is bounded by $2^{(w+1)w}$.
\begin{itemize}
	\item If $t$ is a join node with children $t_1$ and $t_2$, we may compute $\scriptR(t)$ by comparing pairs of records in $\scriptR(t_1)$ and $\scriptR(t_2)$; by inequality~\eqref{EqnRecord}, this may be done in time $\order(2^{2(w+1)(w+d)})$.
	\item If $t$ is an introduce node with child $t'$, we may compute $\scriptR(t)$ by considering records in $\scriptR(t')$ and parent sets of the introduced node $v_0$. Since the number of choices for the latter is bounded by $2^d$, we conclude that $\scriptR(t)$ may be computed in time $\order(2^{(w+1)(w+d) + d})$.
	\item Clearly, if $t$ is a forget node, then $\scriptR(t)$ may be computed in time $\order(2^{(w+1)(w+d)})$.
\end{itemize}

Altogether, we conclude that all records of nodes in the nice tree decomposition may be computed in time $\order(p \cdot 2^{2(w+1)(w+d)})$. Combined with the graphical Lasso preprocessing step for estimating $\scriptM(G_0)$, this leads to an overall complexity of $\order(p^2)$. This may be compared to the runtime of other standard methods for causal inference, including the PC algorithm~\citep{SpiEtal00}, which has computational complexity $\order(p^w)$, and (direct) LiNGAM~\citep{ShiEtal06, ShiEtal11}, which requires time $\order(p^4)$. It has been noted that both the PC and LiNGAM algorithms may be expedited when prior knowledge about the DAG space is available, further highlighting the power of Theorem~\ref{ThmInvDAG} as a preprocessing step for any causal inference algorithm.

%%%%%%%%%%%%%%%%%%%%%%

\section{Discussion}

We have provided a new framework for estimating the DAG corresponding to a linear SEM. We have shown that the inverse covariance matrix of linear SEMs always reflects the edge structure of the moralized graph, even in non-Gaussian settings, and the reverse statement also holds under a mild faithfulness assumption. Furthermore, we have shown that when the error variances are known up to close precision, a simple weighted squared $\ell_2$-loss may be used to select the correct DAG. As a corollary, we have established identifiability for the class of linear SEMs with error variances specified up to a constant multiple. We have proved that our methods are statistically consistent, under reasonable assumptions on the gap between the score of the true DAG and the next best DAG in the model class. A characterization of this gap parameter for various graphical structures is the topic of future work.

We have also shown how dynamic programming may be used to select the best-scoring DAG in an efficient manner, assuming the treewidth of the moralized graph is small. Our results relating the inverse covariance matrix to the moralized DAG provide a powerful method for reducing the DAG search space as a preprocessing step for dynamic programming, and are the first to provide rigorous guarantees for when the graphical Lasso may be used in non-Gaussian settings. Note that the dynamic programming algorithm only uses the information that the true DAG has skeleton lying in the input graph, and does \emph{not} incorporate any information about (a) the fact that the data comes from a linear SEM; or (b) the fact that the input graph exactly equals the moralized DAG. Intuitively, both types of information should place significant constraints on the restricted DAG space, leading to further speedups in the dynamic programming algorithm. Perhaps these restrictions would make it possible to establish a version of dynamic programming for DAGs where the moralized graph has bounded degree but large treewidth.

An important open question concerns scoring candidate DAGs when the diagonal matrix $\Omega_0$ of error variances is unknown. As we have seen, using the weighted squared $\ell_2$-loss to score DAGs may produce a graph that is far from the true DAG when $\Omega_0$ is misspecified. Alternatively, it would be useful to have a checkable condition that would allow us to verify whether a given matrix $\Omega$ will correctly select the true DAG, or to be able to select the true $\Omega_0$ from among a finite collection of candidate matrices.

\acks{We acknowledge all the members of the Seminar f\"{u}r Statistik for providing an immensely hospitable and fruitful environment when PL was visiting ETH, and the Forschungsinstitut f\"{u}r Mathematik at ETH Z\"{u}rich for financial support. We also thank Varun Jog for helpful discussions. PL was additionally supported by a Hertz Fellowship and an NSF Graduate Research Fellowship.}

%%%%%%%%%%%%%%%%%%

\appendix

\section{Graph-theoretic concepts}
\label{AppGraph}

In this Appendix, we review some fundamental concepts in graph theory that we use in our exposition. We begin by discussing junction trees, and then move to the related notion of tree decompositions. Note that these are purely graph-theoretic operations that may be performed on an arbitrary undirected graph.

\subsection{Junction trees}
\label{AppJT}

We begin with the basic junction tree framework. For more details, see~\cite{Lau96} or~\cite{KolFri09}.

For an undirected graph $G = (V,E)$, a \emph{triangulation} is an augmented graph $\widetilde{G} = (V, \widetilde{E})$ that contains no chordless cycles of length greater than three. By classical graph theory, any triangulation $\widetilde{G}$ gives rise to a \emph{junction tree} representation of $G$, where nodes in the junction tree are subsets of $V$ corresponding to maximal cliques of $\widetilde{G}$, and the intersection of any two adjacent cliques $C_1$ and $C_2$ in the junction tree is referred to as a \emph{separator set} $S = C_1 \cap C_2$. Furthermore, any junction tree must satisfy the \emph{running intersection property}, meaning that for any two nodes in the junction tree, say corresponding to cliques $C_j$ and $C_k$, the intersection $C_j \cap C_k$ must belong to every separator set on the unique path between $C_j$ and $C_k$ in the junction tree. The \emph{treewidth} of $G$ is defined to be one less than the size of the largest clique in any triangulation $\widetilde{G}$ of $G$, minimized over all triangulations.

As a classic example, note that if $G$ is a tree, then $G$ is already triangulated, and the junction tree parallels the tree structure of $G$. The maximal cliques in the junction tree are equal to the edges of $G$ and the separator sets correspond to singleton vertices. The treewidth of $G$ is consequently equal to 1.

\subsection{Tree decompositions}
\label{AppTree}

We now highlight some basic concepts of tree decompositions and nice tree decompositions used in our dynamic programming framework. Our exposition follows~\cite{Klo94}.

Let $G = (V,E)$ be an undirected graph. A \emph{tree decomposition} of $G$ is a tree $T$ with node set $W$ such that each node $t \in W$ is associated with a subset $V_t \subseteq V$, and the following properties are satisfied:
\begin{itemize}
	\item[(a)] $\bigcup_{t \in T} V_t = V$;
	\item[(b)] for all $(u,v) \in E$, there exists a node $t \in W$ such that $u,v \in V_t$;
	\item[(c)] for each $v \in V$, the set of nodes $\{t: v \in V_t\}$ forms a subtree of $T$.
\end{itemize}
The \emph{width} of the tree decomposition is $\max_{t \in T} |V_t| - 1$. The \emph{treewidth} of $G$ is the minimal width of any tree decomposition of $G$; this quantity agrees with the treewidth defined in terms of junction trees in the previous section. If $G$ has bounded treewidth, a tree decomposition with minimum width may be constructed in time $\order(|V|)$ (cf.\ Chapter 15 of~\cite{Klo94}).

A \emph{nice tree decomposition} is rooted tree decomposition satisfying the following properties:
\begin{itemize}
	\item[(a)] every node has at most two children;
	\item[(b)] if a node $t$ has two children $r$ and $s$, then $V_t = V_r = V_s$;
	\item[(c)] if a node $t$ has one child $s$, then either
	\begin{itemize}
		\item[(i)] $|V_t| = |V_s| + 1$ and $V_s \subseteq V_t$, or
		\item[(ii)] $|V_s| = |V_t| + 1$ and $V_t \subseteq V_s$.
	\end{itemize}
\end{itemize}
Nodes of the form (b), (c)(i), and (c)(ii) are called \emph{join} nodes, \emph{introduce} nodes, and \emph{forget} nodes, respectively. Given a tree decomposition of $G$ with width $w$, a nice tree decomposition with width $w$ and at most $4|V|$ nodes may be computed in time $\order(|V|)$ (cf.\ Lemma 13.1.3 of~\cite{Klo94}).

%%%%%%%%%%%%%%%%%%%%%%%%%%

\section{Matrix derivations}
\label{AppMatrix}

In this section, we present a few matrix results that are used to prove Theorem~\ref{ThmScore}.

Define a \emph{unit lower triangular (LT)} matrix to be a lower triangular matrix with 1's on the diagonal. Recall that matrices $A$ and $B$ are \emph{permutation similar} if there exists a permutation matrix $P$ such that $A = PBP^T$. Call a matrix \emph{permutation unit LT} if it is permutation similar to a unit lower triangular matrix. We have the following lemma:
\begin{lemma}
	\label{LemLDL}
	Suppose $A$ and $B$ are permutation unit LT matrices, and suppose $AA^T = BB^T$. Then $A = B$.
\end{lemma}

\begin{proof}
	Under the appropriate relabeling, we assume without loss of generality that $A$ is unit LT. There exists a permutation matrix $P$ such that $C \defn PBP^T$ is also unit LT. We have 
\begin{equation}
	\label{EqnLTeq}
	PAA^TP^T = CC^T.
\end{equation}

Let $\pi$ be the permutation on $\{1, \dots, n\}$ such that $P_{i,\pi(i)} = 1$ for all $i$, and $P$ has 0's everywhere else. Define the notation
\begin{equation*}
	\atil_{ij} \defn (PA)_{ij}, \qquad \text{and} \qquad m_{ij} \defn (CC^T)_{ij},
\end{equation*}
and let $\{c_{ij}\}$ denote the entries of $C$. We will make use of the following equalities, which follow from equation~\eqref{EqnLTeq} and the fact that $C$ is unit LT:
\begin{equation}
	\label{EqnMij1}
	\sum_k \atil_{ik} \atil_{jk} \; = \; m_{ij} \; = \; \sum_{k<j} c_{ik} c_{jk} + c_{ij}, \qquad \forall i > j,
\end{equation}
and
\begin{equation}
	\label{EqnMij2}
	\sum_k \atil_{ik}^2 \; = \; m_{ii} \; = \; 1 + \sum_{k < i} c_{ik}^2.
\end{equation}
We now derive the following equality:
\begin{equation}
	\label{EqnAentry}
	\atil_{i,\pi(j)} = c_{ij}, \qquad \forall i, j.
\end{equation}
Note that equation~\eqref{EqnAentry} implies $(PA)P^T = C$, from which it follows that $A = B$.

If $j = i$, we have $\atil_{i,\pi(i)} = 1$ trivially, since $A$ has 1's on the diagonal. For the remaining cases, we induct on $i$. When $i = 1$, we need to show that $\atil_{1,\pi(1)} = 1$ and all other entries in the first row are 0. By equation~\eqref{EqnMij2}, we have
\begin{equation*}
	\sum_k \atil_{1k}^2 = m_{11} = 1.
\end{equation*}
Since $\atil_{1,\pi(1)} = 1$, it is clear that $\atil_{1k} = 0$ for all $k \neq \pi(1)$, establishing the base case.

For the induction step, consider $i > 1$. We first show that $\atil_{i,\pi(j)} = c_{ij}$ for all $j < i$ by a sub-induction on $j$. For $j = 1$, we have by equation~\eqref{EqnMij1} and the base result for $i = 1$ that
\begin{equation*}
	\atil_{i, \pi(1)} = m_{i, 1} = c_{i, 1},
\end{equation*}
which is exactly what we want. For the sub-induction step, consider $1 < j < i$, and suppose $\atil_{i, \pi(\ell)} = c_{i\ell}$ for all $\ell < j$. Note that $\atil_{j, \pi(\ell)} = 0$ for all $\ell > j$ by the outer induction hypothesis. Hence, equation~\eqref{EqnMij1} and the fact that $\atil_{j, \pi(j)} = 1$ gives
\begin{equation}
	\label{EqnCookie}
	\sum_{\ell < j} \atil_{i, \pi(\ell)} \atil_{j, \pi(\ell)} + \atil_{i, \pi(j)} \; = \; m_{ij} \; = \; \sum_{k < j} c_{ik} c_{jk} + c_{ij}.
\end{equation}
Since also $\atil_{j, \pi(\ell)} = c_{j\ell}$ for $\ell < j$ by the outer induction hypothesis, equation~\eqref{EqnCookie} condenses to
\begin{equation*}
	\sum_{\ell < j} c_{i\ell} c_{j\ell} + \atil_{i, \pi(j)} = \sum_{k < j} c_{ik} c_{jk} + c_{ij},
\end{equation*}
from which it follows that $\atil_{i, \pi(j)} = c_{ij}$, as wanted. This completes the inner induction and shows that $\atil_{i, \pi(j)} = c_{ij}$, for all $j < i$. Finally, note that by equation~\eqref{EqnMij2}, we have
\begin{equation*}
	m_{ii} = \sum_k \atil_{ik}^2 = 1 + \sum_{j \neq i} \atil_{i, \pi(j)}^2 \ge 1 + \sum_{j < i} \atil_{i, \pi(j)}^2 = 1 + \sum_{j < i} c_{ij}^2 = m_{ii},
\end{equation*}
implying that we must have $\atil_{i, \pi(j)} = 0$, for all $j > i$. This establishes equation~\eqref{EqnAentry}.

\end{proof}

We also need the following known result (cf.\ Exercise 7.8.19 in~\cite{HorJoh90}). We include a proof for completeness.
\begin{lemma}
	\label{LemHorJoh}
	Suppose $A \in \real^{n \times n}$ is positive definite with $\det(A) = 1$. Then
	\begin{equation*}
		\min\{\tr(AB): B \succ 0 \text{ and } \det(B) = 1\} = n.
	\end{equation*}
\end{lemma}

\begin{proof}
	Consider the singular value decomposition $A = U \Lambda U^*$, and note that $\tr(AB) = \tr(\Lambda(U^*BU))$. Denote $b_{ij} \defn (U^*BU)_{ij}$ and $\lambda_i \defn \Lambda_{ii}$. Then by the AM-GM inequality and Hadamard's inequality, we have
	\begin{equation*}
		\frac{1}{n} \cdot \tr(\Lambda(U^*BU)) \ge \left(\prod_i \lambda_i b_{ii}\right)^{1/n} = \left(\det(A) \cdot \prod_i b_{ii}\right)^{1/n} \ge (\det(U^*BU))^{1/n} = 1,
	\end{equation*}
implying the result.
\end{proof}

Building upon Lemmas~\ref{LemLDL} and~\ref{LemHorJoh}, we obtain the following result:
\begin{lemma}
	\label{LemMinTrace}
	Suppose $A$ and $B$ are $n \times n$ permutation unit LT matrices. Then
	\begin{equation}
		\label{EqnMinTrace}
		\min_B \tr(AA^T B^TB) \ge n,
	\end{equation}
	with equality achieved if and only if $B = A^{-1}$.
\end{lemma}

\begin{proof}
	Write $A' = AA^T$ and $B' = B^TB$, and note that since $\det(A) = \det(B) = 1$, we also have $\det(A') = \det(B') = 1$. Then inequality~\eqref{EqnMinTrace} holds by Lemma~\ref{LemHorJoh}.
	
To recover the conditions for equality, note that equality holds in Hadamard's inequality if and only if some $b_{ii} = 0$ or the matrix $U^*BU$ is diagonal. Note that the first case is not possible, since $U^*BU \succ 0$. In the second case, we see that in addition, we need $b_{ii} = \frac{1}{\lambda_i}$ for all $i$ in order to achieve equality in the AM-GM inequality. It follows that $U^*BU = \Lambda^{-1}$, so $AA^T = B^{-1} B^{-T}$.

Since $A$ and $B$ are permutation unit LT, Lemma~\ref{LemLDL} implies that the last equality can only hold when $B = A^{-1}$.
\end{proof}

%%%%%%%%%%%%%%%%%%%%%

\section{Proofs for population-level results}
\label{AppPop}

In this section, we provide proofs for the remaining results in Sections~\ref{SecCIG} and~\ref{SecFit}.

\subsection{Proof of Lemma~\ref{LemCovs}}
\label{AppCovs}

We first show that $\Omega$ is a diagonal matrix. Consider $j < k$; we will show that $\epsilon_j \condind \epsilon_k$, from which we conclude that
\begin{equation*}
	\E[\epsilon_j \epsilon_k] = \E[\epsilon_j] \cdot \E[\epsilon_k] = 0.
\end{equation*}
Indeed, we have $\epsilon_k \condind (X_1, \dots, X_{k-1})$ by assumption. Since $\epsilon_j = X_j - b_j^T X$ is a deterministic function of $(X_1, \dots, X_j)$, it follows that $\epsilon_k \condind \epsilon_j$, as claimed.

Turning to equations~\eqref{EqnInvElts} and~\eqref{EqnInvDiag}, note that since $B \in \scriptU$, the matrix $(I-B)$ is always invertible, and by equation~\eqref{EqnLinStruct}, we have
\begin{equation}
	\label{EqnCovs}
	\Sigma = (I-B)^{-T} \Omega (I-B)^{-1},
\end{equation}
and
\begin{equation}
	\Theta = \Sigma^{-1} = (I-B) \Omega^{-1} (I-B)^T.
\end{equation}
Then expanding and using the fact that $B$ is upper triangular and $\Omega$ is diagonal, we obtain equations~\eqref{EqnInvElts} and~\eqref{EqnInvDiag}.

\subsection{Proof of Lemma~\ref{LemScoreDAG}}
\label{AppScoreDAG}

Since $G_0 \subseteq G$, we have $\Pa_{G_0}(j) \subseteq \Pa_G(j)$, for each $j$. Furthermore, no element of $\Pa_G(j)$ may be a descendant of $j$ in $G_0$, since this would contradict the fact that $G$ contains no cycles. By the Markov property of $G_0$, we therefore have
\begin{equation*}
	X_j \condind X_{\Pa_G(j) \backslash \Pa_{G_0}(j)} \mid X_{\Pa_{G_0}(j)}.
\end{equation*}
Thus, the linear regression coefficients for $X_j$ regressed upon $X_{\Pa_G(j)}$ are simply the linear regression coefficients for $X_j$ regressed upon $X_{\Pa_{G_0}(j)}$ (and the remaining coefficients for $X_{\Pa_G(j) \backslash \Pa_{G_0}(j)}$ are zero). By Remark~\ref{RemScore}, we conclude that
\begin{equation*}
	B_0 = B_G = \arg\min_{B \in \scriptU_G} \{\score_{\Omega}(B)\},
\end{equation*}
and the uniqueness of $B_0$ follows from the uniqueness of $B_G$.

\subsection{Proof of Theorem~\ref{ThmRatio}}
\label{AppRatio}

From the decomposition~\eqref{EqnScore}, it is easy to see that for any $B \in \scriptU$, we have
\begin{equation}
	\label{EqnSandwich}
	a_{\min} \le \frac{\score_{\Omega_1}(B)}{\score_{\Omega_0}(B)} \le a_{\max},
\end{equation}
simply by comparing individual terms; e.g.,
\begin{align*}
	\frac{1}{(\Omega_1)_{jj}} \cdot \E[(X_j - b_j^T X)^2] & \le \max_j \left\{\frac{1/(\Omega_1)_{jj}}{1/(\Omega_0)_{jj}}\right\} \cdot \frac{1}{(\Omega_0)_{jj}} \cdot \E[(X_j - b_j^T X)^2] \\
	& = a_{\max} \left(\frac{1}{(\Omega_0)_{jj}} \cdot \E[(X_j - B_j^T X)^2]\right).
\end{align*}

Note that if $G \supseteq G_0$, then by Lemma~\ref{LemScoreDAG}, the matrix $B_0$ is the unique minimizer of $\score_{\Omega_1}(B)$ among the class $\scriptU_G$. Now consider $G \not\supseteq G_0$ and $B \in \scriptU_G$. We have
\begin{equation}
	\label{EqnXiBd}
	\left(1 + \frac{\xi}{p}\right) \cdot \score_{\Omega_0}(B_0) = \min_{G' \in \scriptD, \; G' \not\supseteq G_0} \{\score_{\Omega_0}(G')\} \le \score_{\Omega_0}(G) \le \score_{\Omega_0}(B),
\end{equation}
where we have used the definition of the gap~\eqref{EqnXi} and the fact that $\score_{\Omega_0}(B_0) = p$ by Theorem~\ref{ThmScore} in the first inequality. Hence,
\begin{equation*}
	\score_{\Omega_1}(B_0) \le a_{\max} \cdot \score_{\Omega_0}(B_0) \le \frac{a_{\max}}{1 + \xi/p} \cdot \score_{\Omega_0}(B) \le \frac{a_{\max}}{a_{\min}(1 + \xi/p)} \cdot \score_{\Omega_1}(B),
\end{equation*}
where the first and third inequalities use inequality~\eqref{EqnSandwich}, and the second inequality uses inequality~\eqref{EqnXiBd}. By the assumption~\eqref{EqnARatio}, it follows that
\begin{equation*}
	\score_{\Omega_1}(B_0) \le \score_{\Omega_1}(B),
\end{equation*}
as wanted. The statement regarding strict inequality is clear.

\subsection{Proof of Lemma~\ref{LemTwoVar}}
\label{AppTwoVar}

We first consider the case when $r \ge 1$. Then $\frac{a_{\max}}{a_{\min}} = r^2$, so combining Theorem~\ref{ThmRatio} with the expression~\eqref{Eqn2Gap}, we have the sufficient condition
\begin{equation*}
	r^2 \le 1 + \frac{b_0^4}{2(r^4 + b_0^2 r^2)}.
\end{equation*}
Rearranging gives
\begin{equation*}
	b_0^4 - 2r^2(r^2-1) b_0^2 - 2r^4(r^2-1) \ge 0,
\end{equation*}
which is equivalent to
\begin{equation*}
	b_0^2 \ge \frac{2r^2(r^2-1) + \sqrt{(4r^4(r^2-1)^2 + 8r^4(r^2-1))}}{2}.
\end{equation*}
Simplifying yields the desired expression.

If instead $r \le 1$, we have $\frac{a_{\max}}{a_{\min}} = \frac{1}{r^2}$, so the sufficient condition becomes
\begin{equation*}
	\frac{1}{r^2} \le 1 + \frac{b_0^4}{2(r^4 + b_0^2 r^2)},
\end{equation*}
which is equivalent to
\begin{equation*}
	b_0^4 - 2(1-r^2) b_0^2 - 2r^2(1-r^2) \ge 0,
\end{equation*}
or
\begin{equation*}
	b_0^2 \ge \frac{2(1-r)^2 + \sqrt{4(1-r^2)^2 + 8r^2(1-r^2)}}{2}.
\end{equation*}
Simplifying further yields the expression.

\subsection{Proof of Lemma~\ref{Lem3DAG}}
\label{App3DAG}

To compute the gap $\xi$, it is sufficient to consider the graphs in Figure~\ref{Fig3Alt}.
\begin{figure}
	\begin{center}
		\begin{tabular}{cccc}
			\includegraphics[width=3cm]{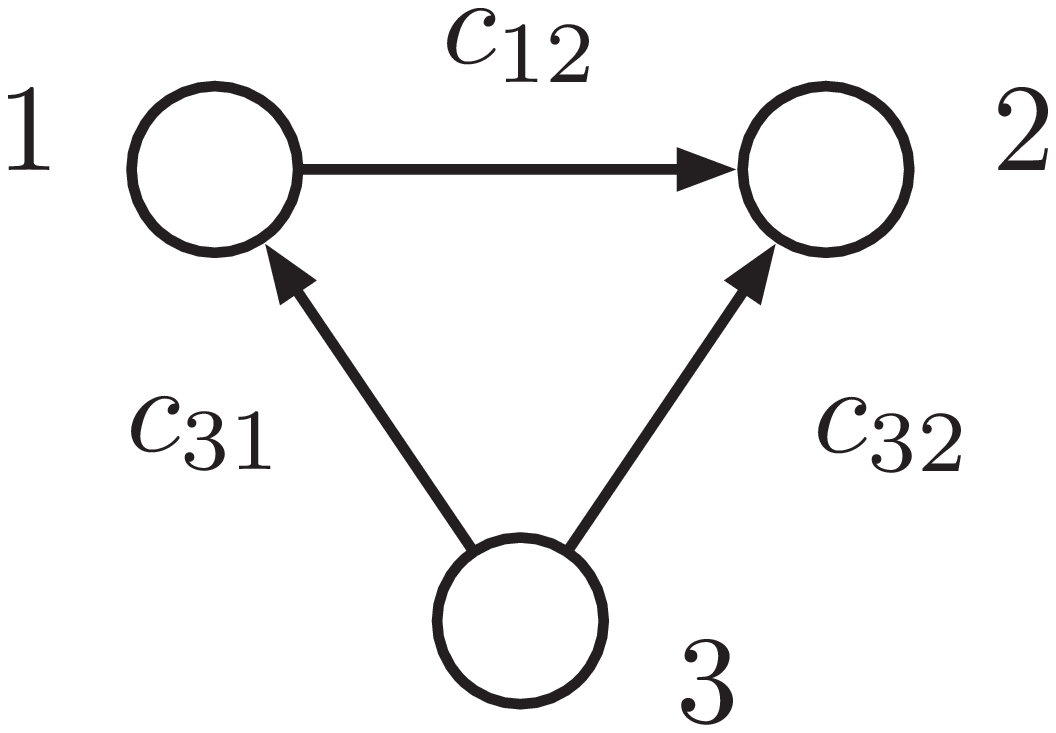} & \includegraphics[width=3cm]{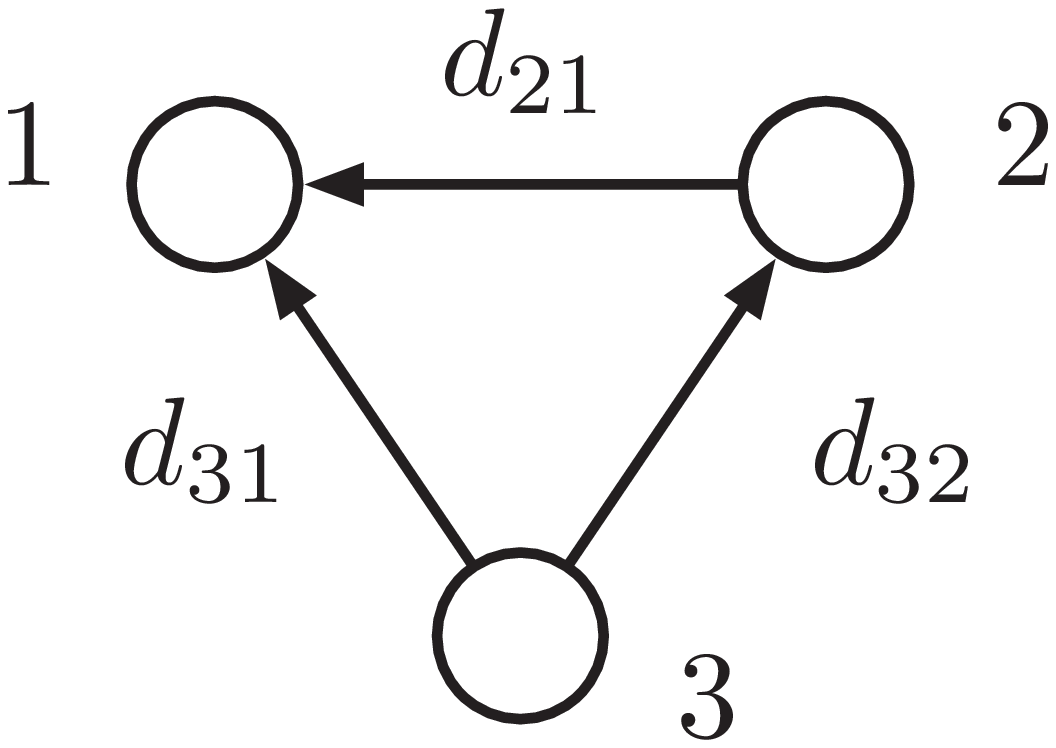} & \includegraphics[width=3cm]{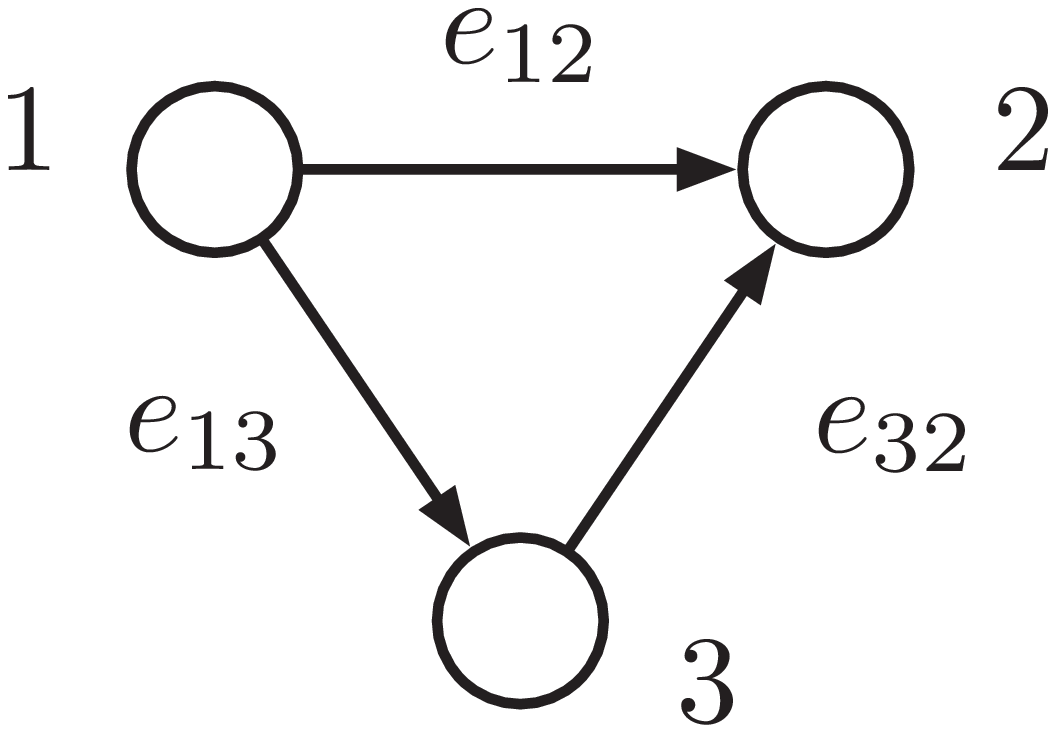} & \includegraphics[width=3cm]{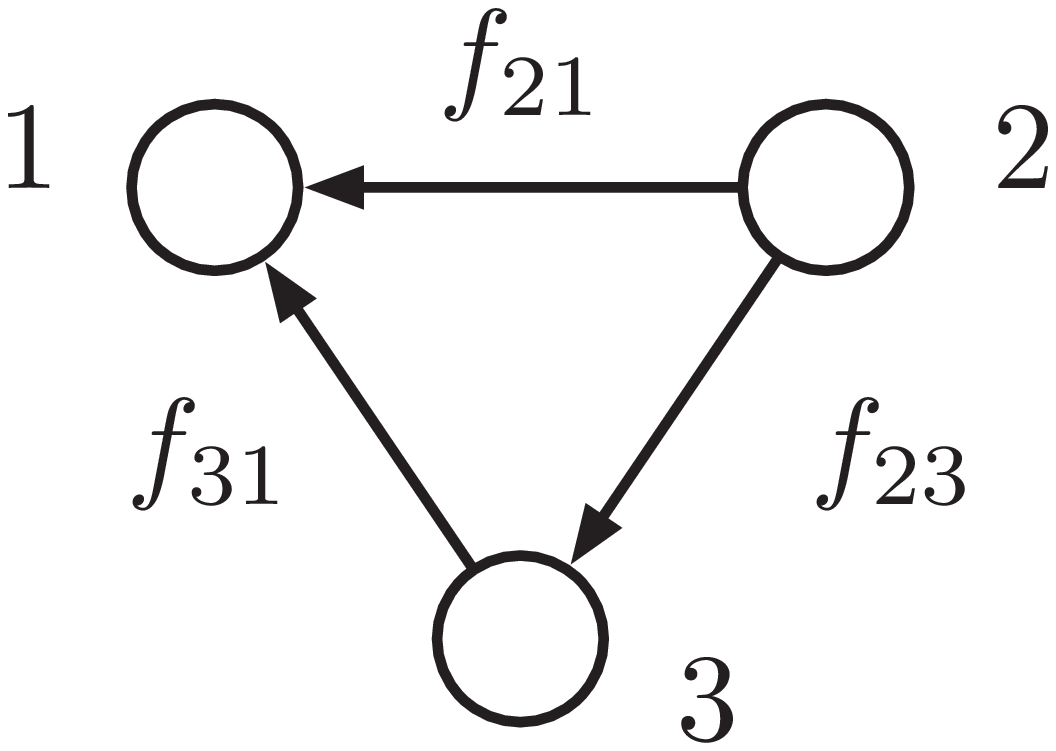} \\
			(a) & (b) & (c) & (d) \\
		\end{tabular}
	\end{center}
	\caption{Alternative DAGs.}
	\label{Fig3Alt}
\end{figure}
Indeed, we have $\score_{\Omega_0}(G) \le \score_{\Omega_0}(G')$ whenever $G' \subseteq G$, so we only need to consider maximal elements in the poset of DAGs not containing $G_0$. Consider the graphs given by autoregression matrices
\begin{equation*}
	C = \left(
	\begin{array}{ccc}
		0 & c_{12} & 0 \\
		0 & 0 & 0 \\
		c_{31} & c_{32} & 0
	\end{array}
	\right), \qquad E = \left(
	\begin{array}{ccc}
		0 & e_{12} & e_{13} \\
		0 & 0 & 0 \\
		0 & e_{32} & 0
	\end{array}
	\right),
\end{equation*}
corresponding to the DAGs in panels (a) and (c) of Figure~\ref{Fig3Alt}. A simple calculation shows that
\begin{multline*}
	\score_{\Omega_0}(C) = 3 + c_{31}^2 b_{13}^2 + c_{32}^2 b_{23}^2 + c_{31}^2 b_{23}^2 \frac{d_2^2}{d_1^2} + \left(c_{12} \frac{d_1}{d_2} + c_{32} b_{13} \frac{d_1}{d_2}\right)^2 \\
	+ \left(c_{31} \frac{d_3}{d_1} - b_{13} \frac{d_1}{d_3}\right)^2 + \left(c_{32} \frac{d_3}{d_2} - b_{23} \frac{d_2}{d_3}\right)^2,
\end{multline*}
which is minimized for
\begin{equation*}
	c_{12} = -b_{13} c_{32}, \qquad c_{31} = \frac{b_{13}}{\frac{d_3^2}{d_1^2} + b_{23}^2 \frac{d_2^2}{d_1^2} + b_{13}^2}, \qquad c_{32} = \frac{b_{23}}{\frac{d_3^2}{d_2^2} + b_{23}^2},
\end{equation*}
leading to
\begin{equation}
	\label{EqnXi1}
	\xi_1 = \min_{c_{12}, c_{31}, c_{32}} \{\score_{\Omega_0}(C) - \score_{\Omega_0}(B_0)\} = \frac{b_{23}^4}{\frac{d_3^4}{d_2^4} + b_{23}^2 \frac{d_3^2}{d_2^2}} + \frac{b_{13}^4 + b_{13}^2 b_{23}^2 \frac{d_2^2}{d_1^2}}{\frac{d_3^4}{d_1^4} + b_{13}^2 \frac{d_3^2}{d_1^2} + b_{23}^2 \frac{d_2^2 d_3^2}{d_1^4}}.
\end{equation}
Similarly, we may compute
\begin{equation*}
	\score_{\Omega_0}(E) = 3 + \left(e_{12} \frac{d_1}{d_2} + e_{32} b_{13} \frac{d_1}{d_2}\right)^2 + \left(e_{13} \frac{d_1}{d_3} - b_{13} \frac{d_1}{d_3}\right)^2 + \left(e_{32} \frac{d_3}{d_2} - b_{23} \frac{d_2}{d_3}\right)^2,
\end{equation*}
which is minimized for
\begin{equation*}
	e_{12} = -b_{13}, \qquad e_{13} = b_{13}, \qquad e_{32} = \frac{b_{23}}{\frac{d_3^2}{d_2^2} + b_{23}^2},
\end{equation*}
leading to
\begin{equation}
	\label{EqnXi2}
	\xi_2 = \min_{e_{12}, e_{13}, e_{32}} \{\score_{\Omega_0}(E) - \score_{\Omega_0}(B_0)\} = \frac{b_{23}^4}{\frac{d_3^4}{d_2^4} + b_{23}^2 \frac{d_3^2}{d_2^2}}.
\end{equation}
Finally, note that the graphs in panels (b) and (d) of Figure~\ref{Fig3Alt} are mirror images of the graphs in panels (a) and (c), respectively. Hence, we obtain
\begin{align*}
	\xi_3 & = \min_{d_{21}, d_{31}, d_{32}} \{\score_{\Omega_0}(D) - \score_{\Omega_0}(B_0) \} = \frac{b_{13}^4}{\frac{d_3^4}{d_1^4} + b_{13}^2 \frac{d_3^2}{d_1^2}} + \frac{b_{23}^4 + b_{13}^2 b_{23}^2 \frac{d_1^2}{d_2^2}}{\frac{d_3^4}{d_2^4} + b_{23}^2 \frac{d_3^2}{d_2^2} + b_{13}^2 \frac{d_1^2 d_3^2}{d_2^4}}, \\
	\xi_4 & = \min_{f_{21}, f_{23}, f_{31}} \{\score_{\Omega_0}(F) - \score_{\Omega_0}(B_0)\} = \frac{b_{13}^4}{\frac{d_3^4}{d_1^4} + b_{13}^2 \frac{d_3^2}{d_1^2}},
\end{align*}
simply by swapping the roles of nodes 1 and 2. Taking $\xi = \min\{\xi_1, \xi_2, \xi_3, \xi_4\}$ then yields the desired result.

\section{Proofs for statistical consistency}
\label{AppStat}

In this Appendix, we provide the proofs for the lemmas on statistical consistency stated in Section~\ref{SecStat}.

\subsection{Proof of Lemma~\ref{LemLowDim}}
\label{AppLowDim}

This result follows from the fact that
	\begin{equation*}
		\|\Thetahat - \Theta_0\|_{\max} \le \opnorm{\Thetahat - \Theta_0}_2,
	\end{equation*}
together with results on the spectral norm of sub-Gaussian covariances and their inverses (see Lemma~\ref{LemSubGSpec} in Appendix~\ref{AppConcentrate}).

\subsection{Proof of Lemma~\ref{LemScoreConc}}
\label{AppScoreConc}

First consider a fixed pair $(j,S)$ such that $S \subseteq N_\Theta(j)$.
We may write
\begin{equation}
	\label{EqnBLP}
	x_j = b_j^T x_S + e_j,
\end{equation}
where $e_j$ has zero mean and is uncorrelated with $x_S$ (and also depends on the choice of $S$). In matrix notation, we have
\begin{equation*}
	\bhat_j = (X_S^T X_S)^{-1} (X_S^T X_j) = (X_S^T X_S)^{-1} X_S^T(X_S b_j + E_j) = b_j + (X_S^T X_S)^{-1} X_S^T E_j,
\end{equation*}
where the second equality follows from equation~\eqref{EqnBLP}. Hence,
\begin{align}
	\label{EqnFhat}
	\sigma_j^2 \cdot \fhat_{\sigma_j}(S) & = \frac{1}{n} \|X_j - X_S \bhat_j\|_2^2 \notag \\
	& = \frac{1}{n} \|X_S(b_j - \bhat_j) + E_j\|_2^2 \notag \\
	& = \frac{1}{n} \|(I - (X_S^T X_S)^{-1} X_S^T) E_j\|_2^2.
\end{align}
By the triangle inequality, we have
\begin{equation}
	\label{EqnKekse}
	\Big| \|(I - (X_S^T X_S)^{-1} X_S^T) E_j\|_2 - \|E_j\|_2 \Big| \le \|(X_S^T X_S)^{-1} X_S^T E_j\|_2 \le \opnorm{(X_S^T X_S)^{-1} X_S^T}_2 \cdot \|E_j\|_2.
\end{equation}
Furthermore,
\begin{align*}
	\opnorm{(X_S^TX_S)^{-1} X_S^T}_2 & = \opnorm{X_S (X_S^TX_S)^{-1}}_2 \\
	& = \sup_{\|v\|_2 \le 1} \left\{v^T (X_S^TX_S)^{-1} X_S^T X_S (X_S^TX_S)^{-1}v \right\}^{1/2} \\
	& = \sup_{\|v\|_2 \le 1} \left\{v^T (X_S^TX_S)^{-1} v\right\}^{1/2} \\
	& = \frac{1}{\sqrt{n}} \opnorm{\left(\frac{X_S^TX_S}{n}\right)^{-1}}_2^{1/2} \\
	& \le \frac{C}{\sqrt{n}} \left(\opnorm{\Sigma_{SS}^{-1}}_2 + 2\sigma^2 \opnorm{\Sigma_{SS}^{-1}}_2^2 \cdot \max\{\delta, \delta^2\}\right)^{1/2},
\end{align*}
with probability at least $1 - 2\exp(-cnt^2)$, where $\delta = c'\sqrt{\frac{|S|}{n}} + c''t$, by Lemma~\ref{LemSubGSpec}. Taking a union bound over all $2^d$ choices for $S$ and $p$ choices for $j$, and setting $t = c\sqrt{\frac{d + \log p}{n}}$, we have
\begin{equation}
	\label{EqnKuchen}
	\opnorm{(X_S^T X_S)^{-1} X_S^T}_2 \le \frac{C'}{\sqrt{n}}, \qquad \forall S \text{ s.t. } S \subseteq N_\Theta(j) \text{ for some } j,
\end{equation}
with probability at least 
\begin{equation*}
	1 - c_1 \exp(-c_2 nt^2) = 1 - c_1 \exp(-c_2 nt^2 + d \log 2 + \log p) \ge 1 - c_1 \exp(-c_2'(d + \log p)).
\end{equation*}
Combining inequalities~\eqref{EqnKekse} and~\eqref{EqnKuchen}, we have the uniform bound
\begin{equation*}
	\left(1 - \frac{C'}{\sqrt{n}}\right)^2 \|E_j\|_2^2 \le \|I - (X_S^TX_S)^{-1} X_S^T)E_j\|_2^2 \le \left(1 + \frac{C'}{\sqrt{n}}\right)^2 \|E_j\|_2^2,
\end{equation*}
w.h.p., which together with equation~\eqref{EqnFhat} implies that
\begin{equation}
	\label{EqnSchoko}
	\left| \sigma_j^2 \cdot \fhat_{\sigma_j}(S) - \frac{1}{n}\|E_j\|_2^2 \right| \le \frac{3C'}{\sqrt{n}} \cdot \frac{1}{n}\|E_j\|_2^2,
\end{equation}
using the fact that
\begin{equation*}
	\max\{1 - (1-a)^2, (1+a)^2 - 1\} = \max\{2a - a^2, 2a + a^2\} = 2a + a^2 \le 3a,
\end{equation*}
for $a = \frac{C'}{\sqrt{n}}$ sufficiently small. Furthermore,
\begin{equation*}
	\frac{1}{n} \E[\|E_j\|_2^2] = \sigma_j^2 \cdot f_{\sigma_j}(S).
\end{equation*}
Note that the $e_j$'s are i.i.d.\ sub-Gaussians with parameter at most $c\sigma^2$, since we may write $e_j = \btil_j^T x$ for the appropriate $\btil_j \in \real^p$, and $\|\btil\|_2$ is bounded in terms of the eigenvalues of $\Sigma$. Applying the usual sub-Gaussian tail bounds, we then have
\begin{equation*}
	\mprob\left(\left|\frac{1}{n} \|E_j\|_2^2 - \frac{1}{n} \E[\|E_j\|_2^2]\right| \ge c\sigma^2 t\right) \le c_1 \exp(-c_2nt^2), \qquad \forall j,
\end{equation*}
and taking a union bound over $j$ and setting $t = c'\sqrt{\frac{\log p}{n}}$ gives
\begin{equation}
	\label{EqnHafer}
	\max_j \left|\frac{1}{n} \|E_j\|_2^2 - \frac{1}{n} \E[\|E_j\|_2^2] \right| \le c_0 \sigma^2 \sqrt{\frac{\log p}{n}},
\end{equation}
with probability at least $1 - c_1 \exp(-c_2 \log p)$. Combining inequalities~\eqref{EqnSchoko} and~\eqref{EqnHafer}, it follows that
\begin{align*}
	\sigma_j^2 |\fhat_{\sigma_j}(S) - f_{\sigma_j}(S)| & \le \left|\sigma_j^2 \cdot \fhat_{\sigma_j}(S) - \frac{1}{n} \|E_j\|_2^2\right| + \left|\frac{1}{n} \|E_j\|_2^2 - \frac{1}{n} \E[\|E_j\|_2^2]\right| \\
	& \le \frac{3C'}{\sqrt{n}} \left(\frac{1}{n} \E[\|E_j\|_2^2] + c_0 \sigma^2 \sqrt{\frac{\log p}{n}}\right) + c_0 \sigma^2 \sqrt{\frac{\log p}{n}} \\
	& \le c_0' \sigma^2 \sqrt{\frac{\log p}{n}},
\end{align*}
with probability at least $1 - c_1 \exp(-c_2 \log p)$.

\subsection{Proof of Lemma~\ref{LemPGap}}
\label{AppPGap}

Combining inequalities~\eqref{EqnScoreConc} and~\eqref{EqnPGap} and using the triangle inequality, we have
\begin{equation}
	\label{EqnScoreBd}
	\left|\scorehat_{\Omega}(G) - \score_{\Omega}(G)\right| \le \sum_{j=1}^p \left|\scorehat_{\sigma_j}(\Pa_G(j)) - \score_{\sigma_j}(\Pa_G(j))\right| < \frac{\xi(\scriptD_\Theta)}{2},
\end{equation}
for all $G \in \scriptD_\Theta$. In particular, for $G_1 \in \scriptD_\Theta$ such that $G_1 \not\supseteq G_0$, we have
\begin{align*}
	\scorehat_\Omega(G_0) & < \score_\Omega(G_0) + \frac{\xi(\scriptD_\Theta)}{2} \\
	& \le \left(\score_\Omega(G_1) - \xi(\scriptD_\Theta)\right) + \frac{\xi(\scriptD_\Theta)}{2} \\
	& < \scorehat_\Omega(G_1),
\end{align*}
where the first and third inequalities use inequality~\eqref{EqnScoreBd} and the second inequality uses the definition of the gap $\xi(\scriptD_\Theta)$. This implies inequality~\eqref{EqnGMin}.

\subsection{Proof of Lemma~\ref{LemXiPrime}}
\label{AppXiPrime}

Consider $G \in \scriptD_\Theta$ with $G \not\supseteq G_0$, and consider $G_1 \supseteq G_0$ such that $\gamma_\Omega(G, G_1)$ is maximized. Note that if $\Pa_G(j) = \Pa_{G_1}(j)$, then certainly,
\begin{equation*}
	\fhat_{\sigma_j}(\Pa_G(j)) - \fhat_{\sigma_j}(\Pa_{G_1}(j)) = 0 = f_{\sigma_j}(\Pa_G(j)) - f_{\sigma_j}(\Pa_{G_1}(j)).
\end{equation*}
Hence,
\begin{equation}
	\label{EqnNuss}
	\left| \big(\scorehat_\Omega(G) - \scorehat_\Omega(G_1)\big) - \big(\score_\Omega(G) - \score_\Omega(G_1)\big) \right| \le |H(G, G_1)| \cdot \frac{\xi'}{2},
\end{equation}
using inequality~\eqref{EqnXiHalf} and the triangle inequality. Furthermore, by inequality~\eqref{EqnXiPrime},
\begin{equation}
	\label{EqnKokos}
	|H(G, G_1)| \cdot \frac{\xi'}{2} \le \frac{\score_\Omega(G) - \score_\Omega(G_1)}{\xi'} \cdot \frac{\xi'}{2} = \frac{\score_\Omega(G) - \score_\Omega(G_1)}{2}.
\end{equation}
Combining inequalities~\eqref{EqnNuss} and~\eqref{EqnKokos} gives
\begin{equation*}
	\scorehat_\Omega(G) - \scorehat_\Omega(G_1) \ge \frac{\score_\Omega(G) - \score_\Omega(G_1)}{2} = \frac{\score_\Omega(G) - \score_\Omega(G_0)}{2} > 0,
\end{equation*}
where the last inequality holds because of the assumption $\xi' > 0$. Hence,
\begin{equation*}
	G \not\in \arg\min_{G \in \scriptD_\Theta} \{\scorehat_\Omega(G)\},	\end{equation*}
implying the desired result.

\subsection{Proof of Lemma~\ref{LemSingleton}}
\label{AppSingleton}

We begin with a simple lemma:
\begin{lemma}
	\label{LemCliques}
	Suppose $\scriptM(G)$ admits a junction tree representation with only singleton separators, and let $C_1, \dots, C_k$ denote the maximal cliques. If $X$ follows a linear SEM over $G$, then the marginal distribution of $X$ over the nodes in any clique $C_\ell$ also follows a linear SEM over $C_\ell$, with DAG structure specified by $G_\ell$, the restriction of $G$ to $C_\ell$. In addition, the autoregression matrix for the marginal SEM is simply the autoregression matrix for the full SEM restricted to the nodes in $C_\ell$.
\end{lemma}

\begin{proof}
	We relabel the nodes of $G$ so that the natural ordering on $\{1, \dots, p\}$ is a topological order. Clearly, this induces a topological order over the nodes of $G_\ell$, as well. Recall that we have equation~\eqref{EqnLinStruct}; i.e., for each $j$,
	\begin{equation}
		\label{EqnLinModel}
		X_j = b_j^T X_{1:j-1} + \epsilon_j, \qquad \text{ where } \epsilon_j \condind (X_1, \dots, X_{j-1}).
	\end{equation}
	For each $j \in C_\ell$, we define
	\begin{equation*}
		\epsilon_j' \defn \epsilon_j + \sum_{k < j, \; k \not\in \Pa_{G_\ell}(j)} b_{kj} X_k,
	\end{equation*}
	and note that
	\begin{equation*}
		X_j = b_j^T X_{\Pa_{G_\ell}(j)} + \epsilon_j',
	\end{equation*}
	where we have abused notation slightly and used $b_j$ to denote the same vector restricted to $\Pa_{G_\ell}(j)$. We claim that
	\begin{equation}
		\label{EqnIndep}
		\epsilon_j' \condind X_{\Pa_{G_\ell}(j)},
	\end{equation}
	for each $j$, implying that the marginal distribution of $X$ over $C_\ell$ follows a linear SEM with the desired properties.
	
	First consider the case when $j$ is not contained in a separator set of the junction tree. Then all neighbors of $j$ must be contained in $C_\ell$, implying that $\Pa_{G_\ell}(j) = \Pa_G(j)$. Since $b_{kj} \neq 0$ only when $k < j$ and $k \in \Pa_G(j)$, this means $\epsilon_j' = \epsilon_j$. The desired independence~\eqref{EqnIndep} follows from equation~\eqref{EqnLinModel} and the simple fact that $\Pa_{G_\ell}(j) \subseteq \{1, \dots, j-1\}$. If instead $j$ is a separator node, then either $\Pa_G(j) \subseteq C_\ell$ or $\Pa_G(j) \cap C_\ell = \emptyset$. In the first case, we again have $\Pa_{G_\ell}(j) = \Pa_G(j)$, so the argument proceeds as before. In the second case, we have $\Pa_{G_\ell}(j) = \emptyset$, so the independence relation~\eqref{EqnIndep} is vacuous; indeed, we have $\epsilon_j' = \epsilon_j + b_j^T X_{\Pa_G(j)} = X_j$. Hence, condition~\eqref{EqnIndep} holds in every case.
\end{proof}

Now consider any $G \in \scriptD_\Theta$ such that $G \not\supseteq G_0$. Let $\{G^\ell\}_{\ell=1}^k$ denote the restrictions of $G$ to the cliques. By Lemma~\ref{LemCliques}, $X$ follows a linear SEM when restricted to the nodes of $C_\ell$; hence, by Lemma~\ref{LemScoreDAG} and Theorem~\ref{ThmScore}, we have
\begin{equation}
	\label{EqnLIneq}
	\score_\Omega(G^\ell_0) \le \score_\Omega(G^\ell),
\end{equation}
with equality if and only if $G^\ell_0 \subseteq G^\ell$. Consider the graph $G_1$ constructed such that $G_1^\ell = G^\ell$ on cliques $C_\ell$ such that inequality~\eqref{EqnLIneq} holds with equality, and $G_1^\ell = G_0^\ell$ otherwise. In particular, we have $G_0^\ell \subseteq G_1^\ell$, for each $\ell$, and
\begin{equation}
	\label{EqnEllCliques}
	\score_\Omega(G^\ell_1) = \score_\Omega(G^\ell_0), \qquad \forall \ell,
\end{equation}
by construction. Note that $G_1$ is always a DAG, but possibly $\scriptM(G_1) \neq \scriptM(G_0)$. However, since $G_0 \subseteq G_1$, we have 
\begin{equation}
	\label{EqnEll}
\score_\Omega(G_0) = \score_\Omega(G_1).
\end{equation}
We also have
\begin{align}
	\label{EqnScoreG}
	\score_\Omega(G) & = \sum_{\ell=1}^k \score_\Omega(G^\ell) - \sum_{r=1}^{k'} (m_r - 1) f_{\sigma_{s_r}}(\emptyset), \\
	\label{EqnScoreG0}
	\score_\Omega(G_0) & = \sum_{\ell=1}^k \score_\Omega(G^\ell_0) - \sum_{r=1}^{k'} (m_r - 1) f_{\sigma_{s_r}}(\emptyset),
\end{align}
where $\{s_r\}_{r=1}^{k'}$ denote the indices of the $k' < k$ separator nodes, and $m_r \defn |\{\ell: s_r \in C_\ell\}|$. This is because both $G$ and $G_0$ have the property that separator nodes only have parents contained in a single clique, so we include an extra term $f_{\sigma_{s_r}}(\emptyset)$ from each adjacent clique not containing $\Pa(s_r)$ in computing the sum. Combining equation~\eqref{EqnScoreG0} with equations~\eqref{EqnEllCliques} and~\eqref{EqnEll}, we must also have
\begin{equation}
	\label{EqnScoreG1}
	\score_\Omega(G_1) = \sum_{\ell=1}^k \score_\Omega(G^\ell_1) - \sum_{r=1}^{k'} (m_r - 1) f_{\sigma_{s_r}}(\emptyset).
\end{equation}
Together with equation~\eqref{EqnScoreG}, this implies
\begin{equation}
	\label{EqnWurst}
	\max_{G_1 \supseteq G_0} \{\gamma_\Omega(G, G_1)\} = \frac{\sum_{\ell=1}^k \left(\score_\Omega(G^\ell) - \score_\Omega(G^\ell_1)\right)}{|H(G,G_1)|}.
\end{equation}
Also note that by Lemma~\ref{LemCliques} and Theorem~\ref{ThmScore}, we have
\begin{equation*}
	\score_\Omega(G^\ell_1) \le \score_\Omega(G^\ell), \qquad \forall \ell,
\end{equation*}
and by assumption,
\begin{equation}
	\label{EqnRivella}
	\frac{\score_\Omega(G^\ell) - \score_\Omega(G^\ell_1)}{|H(G^\ell, G^\ell_1)|} \ge \gamma_\Omega(G^\ell_0), \qquad \forall \ell.
\end{equation}

Finally, reindexing the cliques so that $\{C_1, \dots, C_{k''}\}$ are the cliques such that $G^\ell \neq G^\ell_1$, we have
\begin{equation*}
	H(G,G_1) \subseteq \bigcup_{\ell=1}^{k''} H(G^\ell, G^\ell_1),
\end{equation*}
implying that
\begin{equation}
	\label{EqnApfel}
	|H(G,G_1)| \le \sum_{\ell=1}^{k''} |H(G^\ell, G^\ell_1)|.
\end{equation}
Using the simple fact that $\frac{a_\ell}{b_\ell} \ge \xi$ for all $\ell$, with $a_\ell, b_\ell > 0$, implies $\frac{\sum_\ell a_\ell}{\sum_\ell b_\ell} > \xi$, we conclude from equation~\eqref{EqnWurst} and inequalities~\eqref{EqnRivella} and~\eqref{EqnApfel} that
\begin{equation*}
	\max_{G_1 \supseteq G_0} \{\gamma_\Omega(G, G_1)\} \ge \frac{\sum_{\ell=1}^{k''} \left(\score_\Omega(G^\ell) - \score_\Omega(G_1^\ell)\right)}{\sum_{\ell=1}^{k''} |H(G^\ell, G_1^\ell)|} \ge \min_{1 \le \ell \le k} \gamma_\Omega(G^\ell_0).
\end{equation*}
Since this result holds uniformly over all $G$, we have $\gamma_\Omega(G_0) \ge \min_{1 \le \ell \le k} \gamma_\Omega(G^\ell_0)$, as well.

\subsection{Proof of Lemma~\ref{LemNoisyScoreConc}}
\label{AppNoisyScoreConc}

This proof is quite similar to the proof for the fully-observed case, so we only mention the high-level details here.

We write
\begin{align}
	\label{EqnExpand}
	\sigma_j^2|\ftil_{\sigma_j}(S) - f_{\sigma_j}(S)| & = \left|\left(\Gamhat_{jj} - \Gamhat_{j,S} \Gamhat_{SS}^{-1} \Gamhat_{S,j}^{-1}\right) - \left(\Sigma_{jj} - \Sigma_{j,S} \Sigma_{SS}^{-1} \Sigma_{S,j}\right)\right| \notag \\
	& \le \left|\Gamhat_{jj} - \Sigma_{jj}\right| + \underbrace{\left|\Gamhat_{j,S} \Gamhat_{SS}^{-1} \Gamhat_{S,j} - \Sigma_{j,S} \Sigma_{SS}^{-1} \Sigma_{S,j}\right|}_{A}.
\end{align}
The first term may be bounded directly using inequality~\eqref{EqnDev} and a union bound over $j$:
\begin{equation}
	\label{EqnFirstTerm}
	\mprob\left(\max_j \left|\Gamhat_{jj} - \Sigma_{jj}\right| \ge c\sigma^2 \sqrt{\frac{\log p}{n}}\right) \le c_1' \exp(-c_2' \log p).
\end{equation}
To bound the second term, we use the following expansion:
\begin{align*}
A & \le \left|\Gamhat_{j,S}\left(\Gamhat_{SS}^{-1} - \Sigma_{SS}^{-1}\right) \Gamhat_{S,j}\right| + \left|\Gamhat_{j,S} \Sigma_{SS}^{-1} \left(\Gamhat_{S,j} - \Sigma_{S,j}\right)\right| + \left|\left(\Gamhat_{j,S} - \Sigma_{j,S}\right) \Sigma_{SS}^{-1} \Sigma_{S,j}\right| \\
& \le \opnorm{\Gamhat_{SS}^{-1} - \Sigma_{SS}^{-1}}_2 \|\Gamhat_{S,j}\|_2^2 + \opnorm{\Sigma_{SS}^{-1}}_2 \left(\|\Gamhat_{S,j}\|_2 \|\Gamhat_{S,j} - \Sigma_{S,j}\|_2 + \|\Gamhat_{S,j} - \Sigma_{S,j}\|_2 \|\Sigma_{S,j} \|_2 \right).
\end{align*}
As in the proof of Lemma~\ref{LemSubGSpec} in Appendix~\ref{AppConcentrate}, we may obtain a bound of the form
\begin{equation*}
	\mprob\left(\opnorm{\Gamhat_{SS}^{-1} - \Sigma_{SS}^{-1}}_2 \le c\sigma^2 \left(\sqrt{\frac{d}{n}} + t\right)\right) \le c_1 \exp(-c_2 nt^2),
\end{equation*}
by inverting the deviation condition~\eqref{EqnDev}. Furthermore,
\begin{equation*}
	\|\Gamhat_{S,j} - \Sigma_{S,j}\|_2 \le \opnorm{\Gamhat_{S'S'} - \Sigma_{S'S'}}_2,
\end{equation*}
where $S' \defn S \cup \{j\}$, which may in turn be bounded using the deviation condition~\eqref{EqnDev}. We also have
\begin{equation*}
	\|\Gamhat_{S,j}\|_2 \le \|\Sigma_{S,j}\|_2 + \opnorm{\Gamhat_{S'S'} - \Sigma_{S'S'}}_2.
\end{equation*}
Combining these results and taking a union bound over the $2^d$ choices for $S$ and $p$ choices for $j$, we arrive at a uniform bound of the form
\begin{equation*}
	\mprob\left(A \le c'\sigma^2 \sqrt{\frac{\log p}{n}}\right) \ge 1 - c_1' \exp(-c_2' \log p).
\end{equation*}
Together with inequality~\eqref{EqnFirstTerm} and the expansion~\eqref{EqnExpand}, we then obtain the desired result.

\section{Matrix concentration results}
\label{AppConcentrate}

This Appendix contains matrix concentration results that are used to prove our technical lemmas. We use $\opnorm{\cdot}_2$ to denote the spectral norm of a matrix.
\begin{lemma}
	\label{LemSubGSpec}
	Suppose $\{x_i\}_{i=1}^n \subseteq \real^p$ are i.i.d.\ sub-Gaussian vectors with parameter $\sigma^2$ and covariance $\Sigma$. Then for all $t \ge 0$, we have
	\begin{equation}
		\label{EqnSubG1}
		\mprob\left(\opnorm{\frac{X^TX}{n} - \Sigma}_2 \le \sigma^2 \cdot \max\{\delta, \delta^2\}\right) \ge 1 - 2\exp(-cnt^2),
	\end{equation}
	where $\delta = c' \sqrt{\frac{p}{n}} + c'' t$. Furthermore, if $\frac{X^TX}{n}$ is invertible and
	\begin{equation*}
		\sigma^2 \opnorm{\Sigma^{-1}}_2 \cdot \max\{\delta, \delta^2\} \le \frac{1}{2},
	\end{equation*}
	we have
	\begin{equation}
		\label{EqnSubG2}
		\mprob\left(\opnorm{\left(\frac{X^TX}{n}\right)^{-1} - \Sigma^{-1}}_2 \le 2 \sigma^2 \opnorm{\Sigma^{-1}}_2^2 \cdot \max\{\delta, \delta^2\} \right) \ge 1 - 2\exp(-cnt^2).
	\end{equation}
\end{lemma}

\begin{proof}
For inequality~\eqref{EqnSubG1}, see Remark 5.40 of~\cite{Ver12}. For inequality~\eqref{EqnSubG2}, we use the matrix expansion
\begin{equation*}
	(A + \Delta)^{-1} = (A(I + A^{-1} \Delta))^{-1} = (I+A^{-1}\Delta)^{-1} A^{-1} = A^{-1} + \sum_{k=1}^\infty (-1)^k (A^{-1} \Delta)^k A^{-1},
\end{equation*}
valid for any matrices $A$ and $\Delta$ such that $A$ and $A+\Delta$ are both invertible and the series converges. By the triangle inequality and multiplicativity of the spectral norm, we then have
\begin{align*}
	\opnorm{(A+\Delta)^{-1} - A^{-1}}_2 & \le \sum_{k=1}^\infty \opnorm{(A^{-1} \Delta)^k A^{-1}}_2 \\
	& \le \opnorm{A^{-1}}_2 \cdot \sum_{k=1}^\infty \opnorm{A^{-1} \Delta}_2^k \\
	& = \frac{\opnorm{A^{-1}}_2 \cdot \opnorm{A^{-1} \Delta}_2}{1 - \opnorm{A^{-1} \Delta}_2} \\
	& \le \frac{\opnorm{A^{-1}}_2^2 \cdot \opnorm{\Delta}_2}{1 - \opnorm{A^{-1} \Delta}_2}.
\end{align*}
We now take $A = \Sigma$ and $\Delta = \frac{X^TX}{n} - \Sigma$. By the assumption and inequality~\eqref{EqnSubG1}, we have
\begin{equation*}
	\opnorm{A^{-1} \Delta}_2 \le \opnorm{A^{-1}}_2 \cdot \opnorm{\Delta}_2 \le \frac{1}{2},
\end{equation*}
implying that
\begin{equation*}
	\opnorm{(A+\Delta)^{-1} - A^{-1}}_2 \le 2 \opnorm{A^{-1}}_2^2 \cdot \opnorm{\Delta}_2.
\end{equation*}
This gives the result.	
\end{proof}

%%%%%%%%%%%%%%%%%%%%%%%%%%%

\bibliography{refs.bib}

\end{document}

%% file: RatbertMacros.tex
\newlength{\widebarargwidth}
\newlength{\widebarargheight}
\newlength{\widebarargdepth}

\newenvironment{carlist}
 {\begin{list}{$\bullet$}
 {\setlength{\topsep}{0in} \setlength{\partopsep}{0in}
  \setlength{\parsep}{0in} \setlength{\itemsep}{\parskip}
  \setlength{\leftmargin}{0.07in} \setlength{\rightmargin}{0.08in}
  \setlength{\listparindent}{0in} \setlength{\labelwidth}{0.08in}
  \setlength{\labelsep}{0.1in} \setlength{\itemindent}{0in}}}
 {\end{list}}

\newcommand{\bcar}{\begin{carlist}}
\newcommand{\ecar}{\end{carlist}}

\makeatletter
\long\def\@makecaption#1#2{
        \vskip 0.8ex
        \setbox\@tempboxa\hbox{\small {\bf #1:} #2}
        \parindent 1.5em  %% How can we use the global value of this???
        \dimen0=\hsize
        \advance\dimen0 by -3em
        \ifdim \wd\@tempboxa >\dimen0
                \hbox to \hsize{
                        \parindent 0em
                        \hfil 
                        \parbox{\dimen0}{\def\baselinestretch{0.96}\small
                                {\bf #1.} #2
                                } 
                        \hfil}
        \else \hbox to \hsize{\hfil \box\@tempboxa \hfil}
        \fi
        }
\makeatother

\long\def\comment#1{}

\makeatletter
\def\@cite#1#2{[\if@tempswa #2 \fi #1]}
\makeatother

\makeatletter
\long\def\barenote#1{
    \insert\footins{\footnotesize
    \interlinepenalty\interfootnotelinepenalty 
    \splittopskip\footnotesep
    \splitmaxdepth \dp\strutbox \floatingpenalty \@MM
    \hsize\columnwidth \@parboxrestore
    {\rule{\z@}{\footnotesep}\ignorespaces
      #1\strut}}}
\makeatother
\newcommand{\bit}{\begin{itemize}}
\newcommand{\eit}{\end{itemize}}
\newcommand{\ben}{\begin{enumerate}}
\newcommand{\een}{\end{enumerate}}

\newcommand{\bear}{\begin{eqnarray}}
\newcommand{\eear}{\end{eqnarray}}

\newcommand{\supp}{\ensuremath{\operatorname{supp}}}

\newcommand{\order}{{\mathcal{O}}}

\newcommand{\beq}{\begin{quotation}}
\newcommand{\enq}{\end{quotation}}

\newcommand{\estart}{\begin{equation}}
\newcommand{\eend}{\end{equation}}

\newcommand{\defn}{\ensuremath{:  =}}

\newcommand{\bec}{\begin{center}}
\newcommand{\enc}{\end{center}}

\newcommand{\beit}{\begin{itemize}}
\newcommand{\enit}{\end{itemize}}

\newcommand{\been}{\begin{enumerate}}
\newcommand{\enen}{\end{enumerate}}

\newcommand{\comsl}{\begin{slide}}
\newcommand{\comspor}{\begin{slide*}}

\newcommand{\comsld}[2]{\begin{slide}[#1,#2]}
\newcommand{\comspord}[2]{\begin{slide*}[#1,#2]}

\newcommand{\mendsl}{\end{slide}}
\newcommand{\mendspo}{\end{slide*}}

\newcommand{\real}{\ensuremath{{\mathbb{R}}}}

\DeclareMathOperator{\diag}{diag}
\DeclareMathOperator{\var}{var}

\DeclareMathOperator{\cov}{cov}